\documentclass[sigconf]{acmart}

\copyrightyear{2026}
\acmYear{2026}
\setcopyright{cc}
\setcctype{by}
\acmConference[ICSE '26]{2026 IEEE/ACM 48th International Conference on Software Engineering}{April 12--18, 2026}{Rio de Janeiro, Brazil}
\acmBooktitle{2026 IEEE/ACM 48th International Conference on Software Engineering (ICSE '26), April 12--18, 2026, Rio de Janeiro, Brazil}
\acmPrice{}
\acmDOI{10.1145/3744916.3787842}
\acmISBN{979-8-4007-2025-3/2026/04}

\usepackage{array}
\usepackage{url}             
\usepackage{colortbl}        
\usepackage{subcaption}      
\usepackage[inline]{enumitem}  
\usepackage{multirow}        
\usepackage{textcomp}        
\usepackage{xspace}          
\usepackage{soul}            
\usepackage{stmaryrd}

\newtheorem{dfn}{Definition}
\newtheorem{proposition}{Proposition}
\usepackage{algorithm, algpseudocode}
\usepackage{siunitx}        
\usepackage{makecell}       
\usepackage{fontawesome5}
\usepackage{pifont}

\newcommand{\D}{\mathcal{D}}
\newcommand{\M}{\mathcal{M}}
\newcommand{\mnistSymb}{${}_{\text{\faSignature}}$\faMarker\xspace}
\newcommand{\quoraSymb}{\faQuora\xspace}
\newcommand{\cifarSymb}{\faImages\xspace}
\newcommand{\tsignSymb}{\faMapSigns\xspace}
\newcommand{\sdcarSymb}{\faCarCrash\xspace}

\newcommand{\sys}{\textsc{TestifAI}\xspace}

\begin{document}

\title[\textsc{TestifAI}: 
Tomography-Based Testing
for Deep Learning Systems]{\textsc{TestifAI}: 
Tomography-Based Testing
for 
Deep Learning Systems}

\author{Arooj Arif}
\affiliation{%
  \institution{Northeastern University London}
  \city{London}
  \country{United Kingdom}
}
\email{arooj.arif@nulondon.ac.uk} 

\author{Tobias Hartung}
\affiliation{%
  \institution{Northeastern University London}
  \city{London}
  \country{United Kingdom}
}
\email{tobias.hartung@nulondon.ac.uk}  

\author{Elena Botoeva}
\affiliation{%
  \institution{University of Kent}
  \city{Canterbury}  
  \country{United Kingdom}
}
\email{e.botoeva@kent.ac.uk}  

\author{Alexandros Koliousis}
  \affiliation{%
  \institution{Northeastern University London}
  \city{London}
  \country{United Kingdom}
}
\email{alexandros.koliousis@nulondon.ac.uk}  

\begin{abstract}
As AI systems are increasingly deployed 
in safety-critical 
application 
domains (e.g., autonomous driving), 
associated risks increase
too.
Deep learning models underlying modern AI systems, 
therefore, must undergo thorough testing 
to ensure 
their 
correct behaviour.
A single robustness test involves thousands 
of inferences to empirically verify 
if a model's outputs remain
stable
under 
a
bounded perturbation 
of its inputs.
However, existing testing frameworks 
lack the means to 
systematically explore
and summarise robustness
across 
a  combinatorial 
space of perturbations. 

We propose \sys, 
a deep learning testing framework for 
efficient and accurate estimation of robustness against combinations of perturbations.
\sys
enables
users to specify 
operational conditions 
as structured spaces of semantic input perturbations
(e.g., image blur, brightness and zoom)
and 
discrete severity levels
(e.g., low, medium and high).
Users can query
model robustness
for any combination
(e.g., ``low blur, high brightness, and medium zoom'').
To achieve efficiency and accuracy, \sys
introduces \emph{partial model tomography}, a novel approach to
reconstructing model behaviour
in a multi-perturbation space
from tests
that apply only a small 
number of perturbations (lower-order projections).
To estimate robustness against at least three perturbations, \sys trains an auxiliary model on the results of tests involving up to two perturbations only, 
avoiding execution of an exponential number
of tests.
Our experiments
on five image and language
classification tasks
show that 
\sys can predict higher-order (3 and 4 perturbations) 
test outcomes
from low-order (1 and 2 perturbations) observations 
with an aggregate robustness estimation error of less than 7\%, 
while
reducing the number of inferences by 60--80\%.
\end{abstract}



\maketitle

\section{Introduction}

Deep learning models 
are now 
integral 
to many
modern software systems,
building on their
success in computer vision
and language understanding.
tasks.
They are increasingly deployed in
high-stakes settings, powering
autonomous vehicles (e.g., 
for 
traffic sign classification~\cite{olivesgatech_CURETSD}
and lane detection~\cite{chen2024end}),
AI chatbots 
(e.g., 
for 
translation~\cite{bahdanau2015neural}, 
question
answering~\cite{wang2018glue}, and 
mental health support~\cite{guo2025chatbot})
and other real-world applications.

Deep learning models are notoriously sensitive 
(i.e., \emph{not} robust) 
to 
input
perturbations~\cite{goodfellow2015explaining,hendrycks2019benchmarking,machado2021adversarial}.
For example, vision models often misclassify images when exposed
to changes in
lighting, 
geometric distortions, or adverse weather
conditions~\cite{hendrycks2019benchmarking}. Similarly,
language models are vulnerable to 
misspellings~\cite{pruthi2019combating}, 
character flips~\cite{ebrahimi2018hotflip}, 
or paraphrases~\cite{morris2020textattack}.
Failures are not limited to 
``in vitro'' 
benchmarks: perception systems in self-driving cars 
have failed to 
identify 
lane markings in poor weather, 
contributing to accidents~\cite{reuters2024nhtsa}; and
chatbots have been
manipulated into producing 
harmful or inappropriate responses through subtle 
prompt variations~\cite{guo2025chatbot}.

As with any software, 
deep learning models 
must be thoroughly tested 
under 
their intended
operating conditions.
This includes 
testing 
assumptions about the training
data---e.g., is the 
input distribution $P(X)$ informative of 
the classification task $P(Y \vert X)$?---as 
well as
inductive biases of the 
model 
architecture---e.g., 
robustness to small image translations in convolutional models
and to word permutations in attention-based models~\cite{Goyal2022}.
In practice, this means checking 
whether a model maintains 
correct predictions
under input perturbations.
The simplest robustness
test applies 
a single semantic transformation (e.g., blurring images with a Gaussian kernel of radius 3, or substituting two words in sentences)
and verifies that predictions remain stable across all examples. 
However, such tests are inherently local: 
they evaluate one axis 
of variation at a time. 
Model performance, however, 
is often affected in new ways by multiple, 
interacting perturbations,
whose combined effects cannot be inferred from independent tests~\cite{hendrycks2020augmix, mu2019mnistc, chandrasekaran2021combinatorial}.
What is needed is a systematic method 
to explore and
summarise robustness 
in combinatorially rich perturbation spaces.
 
Existing deep learning testing techniques
either 
overlook interaction effects,
or 
lack mechanisms
to explore them.
They can be broadly grouped into four categories:

\begin{enumerate*}[label=(\textit{\roman*})]
\item 
\emph{Static robustness benchmarks}
define a fixed suite of semantic perturbations 
that simulate realistic deployment conditions (e.g., 
CIFAR-10-C~\cite{hendrycks2019benchmarking} 
and \textsf{TextFlint}~\cite{wang2021textflint}).
They typically discretise the perturbation intensity 
into severity levels (e.g., low, medium, high).
However, perturbations---whether simple (e.g., ``blur'') 
or complex (e.g., ``fog'')---are treated as atomic units.
As a result, users lack control over how 
perturbations interact, and 
cannot inspect or adjust 
their compositional structure.
\item 
\emph{Test prioritisation methods} 
use uncertainty or diversity estimates to 
select inputs 
that are more likely 
to trigger model failures~\cite{ma2021test,feng2020deepgini,gao2022adaptive}.
They can accelerate 
robustness evaluation by focusing on inputs 
that are most 
informative, especially when testing
worst-case rather than average model behaviour. But
they operate on a single-perturbation setting (i.e., one test)
and do not reason about interactions across perturbations.

\item \emph{Neuron coverage methods} attempt to quantify 
test adequacy using internal model representations, 
such as activation patterns across neurons~\cite{pei2017deepxplore,ma2018deepgauge,xie2019deephunter},
or generate new test examples by 
perturbing inputs in the latent space~\cite{dola2024cit4dnn}. 
However, coverage metrics 
do not align with semantically 
meaningful perturbations, 
nor do they account for compositional 
interactions between them.

\item \emph{Combinatorial methods} 
compose input perturbations, 
either stochastically~\cite{hendrycks2020augmix}
or exhaustively~\cite{chandrasekaran2021combinatorial},
to improve or evaluate model robustness.
However, they provide no systematic way 
to assess robustness across the full perturbation space.
\end{enumerate*}

In this paper, we introduce \sys, a test framework for deep learning models
that enables exploration of their robustness across
multi-perturbation scenarios. 
The key contribution
is \emph{partial model tomography}. 
Rather than exhaustively evaluating all combinations 
of multiple perturbations and severity levels, 
\sys executes only a subset of tests---specifically, 
those involving one or two perturbations at a time---and 
learns to predict the rest.
In other words, it estimates higher-order 
robustness by reconstructing 
the full perturbation space 
from its lower-order projections.

We demonstrate the efficacy of \sys by training 
a random forest on sampled first- and second-order tests, 
and use it to approximate the robustness for untested, 
higher-order configurations. We use five realistic
benchmarks---four vision
and one language classification tasks---to evaluate it and
show that \sys can estimate robustness 
across the full test space, including 
all 3-way and 4-way combinations, with 
minimal approximation error.

\sys is a complete test framework and makes two further contributions:
(\emph{i}) it supports interactive analysis through a Boolean query interface, 
allowing
users to query robustness over perturbation types and severity levels; 
(\emph{ii}) its implementation includes 
an adaptive sampling strategy 
with early stopping that reduces the number 
of inferences required 
per test 
by detecting 
convergence.

\section{Testing Deep Learning Models}

Deep learning classifiers---the focus of our work---implicitly learn 
decision boundaries from data, namely input features $X$  
and 
their corresponding 
labels $Y$, by modelling the
conditional distribution 
$P(Y~\vert~X)$. 
Testing 
evaluates
a model's
statistical behavior 
when 
the 
input distribution $P(X)$ shifts
due to one or more perturbations:
we apply 
controlled changes to inputs $x \in X$ that 
are expected to preserve their true label $y \in Y$.
A \emph{robust classifier} should maintain consistent predictions for perturbed versions 
of the same example $(x, y)$.
However, perturbations can also expose cases that change 
the model's approximation of $P(Y~\vert~X)$, 
revealing the brittleness of the
classifier's learned decision boundary.

\subsection{Metamorphic tests}
\label{Metamorphic tests}

We interpret perturbations as instances of \emph{metamorphic relations}---expected invariances under small, label-preserving transformations~\cite{togru2024enhancing}.
We consider a deep learning model $\mathcal{M}$, a labelled test set $\mathcal{D}=\{(x_1,y_1), \dots, (x_n, y_n)\}$, and a set of parameterised perturbations $\mathcal{P} = \{ p_1, \dots, p_n\}$, where each $p_i$ is a perturbation function (e.g., image rotation). 
We further define the corresponding \emph{severity level} sets $S_1, \dots, S_n$ (e.g., $S_i=\{0,1,2,3,4,5\}$, where 0 denotes no rotation, 
1 rotation up to $10^\circ$, etc.) and the set of severity combinations $\mathcal{S} = S_1\times S_2\times\dots\times S_n$.

\begin{dfn}[Robustness]
\normalfont
Let \(\boldsymbol{\sigma}\in \mathcal{S}\) denote 
a
configuration of perturbations
(e.g., \(\boldsymbol{\sigma}=(0,4,2)\) or \(\boldsymbol{\sigma}=(1,2,5)\) for $n=3$), and let \( \pi_{\boldsymbol{\sigma}}(x) \) be the 
composite perturbation of input $x$ 
induced by \( \boldsymbol{\sigma} \).
The {\emph{robustness}} score \(\mathrm{r}_{\boldsymbol{\sigma}}\) is 
the fraction of inputs in $\D$ on which \(\M\) remains 
correct under \(\boldsymbol{\sigma}\):
\begin{equation}
  \label{eq:local-robustness}
  \mathrm{r}_{\boldsymbol{\sigma}}
  =\frac{1}{|\mathcal{D}|}\sum_{i=1}^{|\mathcal{D}|}
  \mathbf{1}\bigl[\mathcal{M}\bigl(\pi_{\boldsymbol{\sigma}}(x_i)\bigr)=y_i\bigr]\,.
\end{equation}
Here \(\mathbf{1}[\cdot]\) is an indicator function that returns 1 if the condition in the brackets is satisfied and 0 otherwise. 
We 
view
\autoref{eq:local-robustness} as a {\emph{system-level metamorphic test}}, or simply a \emph{test}.
We may refer to a test by its 
configuration \({\boldsymbol{\sigma}}\)
or 
by its result \(\mathrm{r}_{\boldsymbol{\sigma}}\).
\end{dfn}

\begin{dfn}[Aggregate Robustness]
\normalfont
Given a suite $\Theta$ of \(N\) tests, \(\Theta=\{\boldsymbol{\sigma}_i\}_{i=1}^N\), the
{\emph{aggregate robustness}} score $\mathrm{R}(\Theta)$ 
of $\Theta$ is the average 
of the robustness scores of all tests in $\Theta$:
\begin{equation}
  \label{eq:global-robustness}
  \mathrm{R}(\Theta)
  =\frac{1}{N}\sum_{i=1}^{N} \mathrm{r}_{\boldsymbol{\sigma}_i}\,.
\end{equation}
A high aggregate robustness score \(\mathrm{R}(\Theta)\) indicates that, on average, \(\mathcal{M}\) has high accuracy across the entire set of tested perturbations.
\end{dfn}

\subsection{Test spaces}

Deep learning models deployed 
in real-world environments
must remain robust under a wide range of perturbations.
For example, vision models in modern cars
must handle varied lighting 
conditions
(e.g., driving through tunnels),
adverse weather (e.g., rain or snow),
and unpredictable human behaviour~\cite{beigi2025impact,geiger2013kitti}.
Similarly, language models must cope with
word substitutions, 
negations, and 
typographic errors~\cite{ribeiro2020beyond,wang2021textflint}.

Several benchmarks 
simulate
real-world perturbations for 
specific
application domains.
MNIST-C~\cite{mu2019mnistc}, for example,
defines fifteen perturbation types for hand-written digit recognition ({\footnotesize \mnistSymb}). Similarly,
\(\{\text{CIFAR-10}, \text{ImageNet}\}\)-C~\cite{hendrycks2019benchmarking},
DeepTest~\cite{tian2018deeptest},
CURE-TSD~\cite{cure-tsr},
and
CheckList~\cite{ribeiro2020beyond}
define
perturbations
for {image} classification ({\footnotesize \cifarSymb}),
self-driving scenarios ({\footnotesize \sdcarSymb}),
traffic sign recognition ({\footnotesize \tsignSymb}),
and
question answering ({\footnotesize \quoraSymb}), respectively.

Besides 
realism,
these benchmarks share
another desirable property: 
most assign discrete 
severity levels---typically 
from 1 to 5---to each perturbation, 
making them suitable 
for systematic testing. 
Some perturbations are relatively \emph{atomic}, 
applying simple transformations to inputs
(e.g., brightness or blur), while
others represent more complex scenarios 
that 
combine multiple effects (e.g., fog or glare). 
Although benchmarks may define \emph{composite} perturbations manually, 
it is infeasible to anticipate 
all combinations 
in advance.
A lane-detection model, for instance, 
may behave correctly under fog or motion blur individually, 
but fail when both are present.
We must therefore move beyond 
isolated transformations 
and consider how perturbations interact~\cite{mintun2021interaction, gao2019automated, dalva2023benchmarking, UAV-C, hendrycks2020augmix, hendrycks2022pixmix, yun2019cutmix, calian2021defending}.

\emph{Why combine more than two perturbations?}
Real-world image and text inputs 
rarely vary along only a single axis.
For example, Multi-Weather City~\cite{mușat2021multi} 
combines three or more weather effects 
to model realistic conditions, and 
ReCode~\cite{wang2023recode} defines over thirty 
semantic-preserving text transformations 
that commonly co-occur in mixed 
natural language-code inputs. 
Higher-order perturbations also arise in training 
(e.g., AugMix~\cite{hendrycks2020augmix} 
augments images with multiple ImageNet-C corruptions) 
and in neural architecture design 
(e.g., spatial transformers~\cite{jaderberg2015spatial} 
apply compound geometric shifts). 
These examples motivate testing 
\emph{higher-order} perturbation combinations 
rather than isolated transformations.
We further discuss the validity of 
composite perturbations in \S\ref{sec:validity}.

\emph{The challenge is to define a test space that enables the exploration of structured combinations of perturbations.}

\subsection{Combinatorial testing}

While testing combinations of perturbations 
is essential, 
it introduces a classic challenge---\emph{combinatorial explosion}.
Even for modest settings,
the number 
of possible tests 
grows exponentially
with the number of perturbations, 
making
exhaustive testing infeasible.

\emph{Combinatorial Interaction Testing} (CIT)~\cite{luo2024beyond,bombarda2024design,chandrasekaran2021combinatorial,krishnan2007combinatorial}
is a principled approach 
to address
the combinatorial explosion problem.
Instead of testing all combinations of
$k$ perturbations, 
CIT constructs a minimal set of 
tests---a \emph{covering array}---that
guarantees $t$-way coverage
of all interactions among any
$t < k$ perturbations.
For example,
a $2$-way covering array over three perturbations
ensures
that all severity-level 
combinations across any pair of perturbations 
will be tested, leaving the remaining perturbations unconstrained.
Higher $t$ increases coverage but also
the computational cost~\cite{chandrasekaran2021combinatorial}.

CIT assumes 
that most failures 
are triggered by low-order interactions 
(i.e., involving only a small number of perturbations)
and 
seeks to expose them through systematic coverage. 
In deep learning, CIT has been 
used to generate
diverse test inputs, 
either by sampling $2$-way combinations 
of real-world perturbations~\cite{chandrasekaran2021combinatorial} 
or by perturbing latent representations 
to generate new inputs~\cite{dola2024cit4dnn}.
However, these methods focus on diversity and coverage,
not on
modelling or predicting 
how interactions affect model behaviour.

\emph{The challenge is to predict model behaviour 
under higher-order (i.e., multi-perturbation) combinations.}

\section{Partial Model Tomography}

\begin{figure*}[t]
    \centering
    \begin{subfigure}[b]{0.15\textwidth}
        \includegraphics[width=\textwidth]{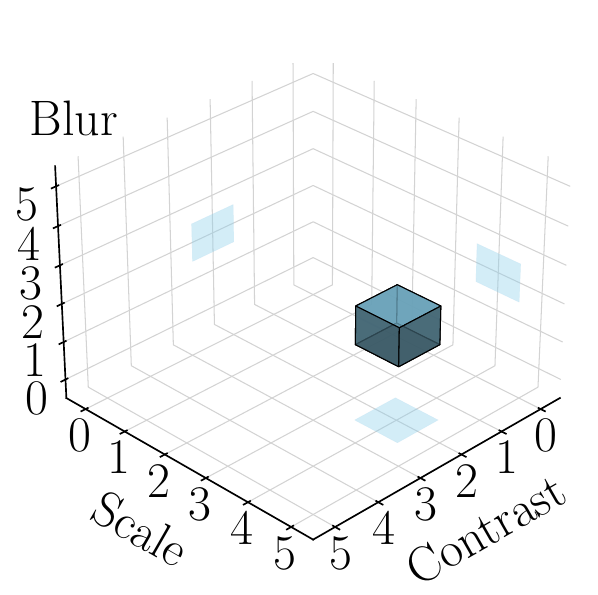}
        \caption{$\boldsymbol{\sigma}=(2,4,2)$}
        \label{fig:three-a}
    \end{subfigure}
    \begin{subfigure}[b]{0.15\textwidth}
        \centering
         \includegraphics[width=0.8\textwidth]{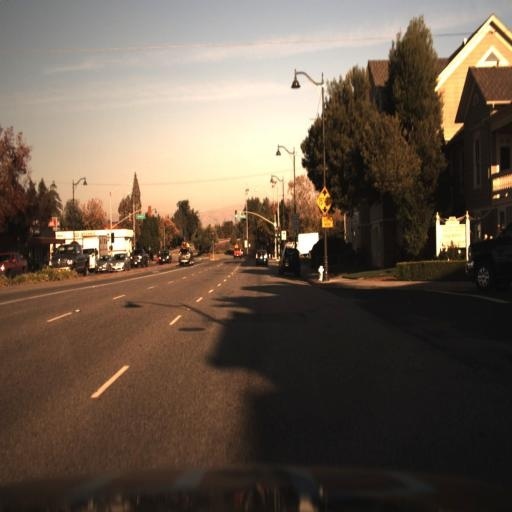}
        \caption{Original image}
        \label{fig:three-b}
    \end{subfigure}
    \begin{subfigure}[b]{0.15\textwidth}
        \centering
        \includegraphics[width=0.8\textwidth]{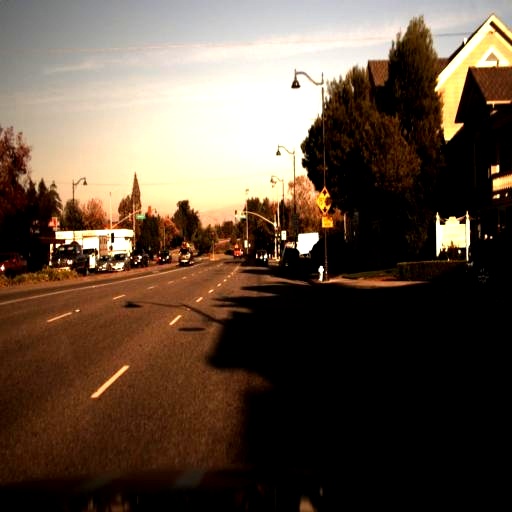}
        \caption{Contrasted}
        \label{fig:three-c}
    \end{subfigure}
    \begin{subfigure}[b]{0.15\textwidth}
        \centering
        \includegraphics[width=0.8\textwidth]{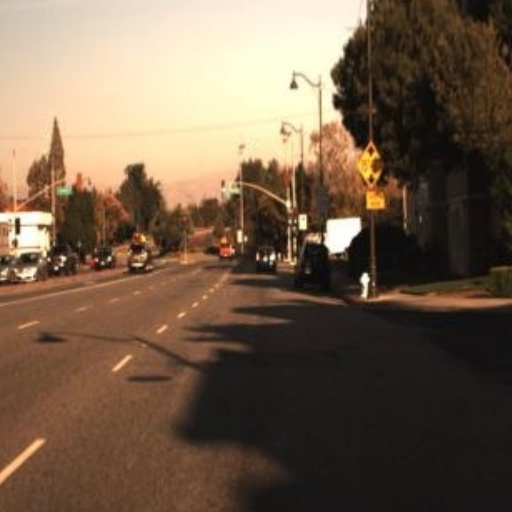}
        \caption{Scaled}
        \label{fig:three-d}
    \end{subfigure}
   \begin{subfigure}[b]{0.15\textwidth}
        \centering
        \includegraphics[width=0.8\textwidth]{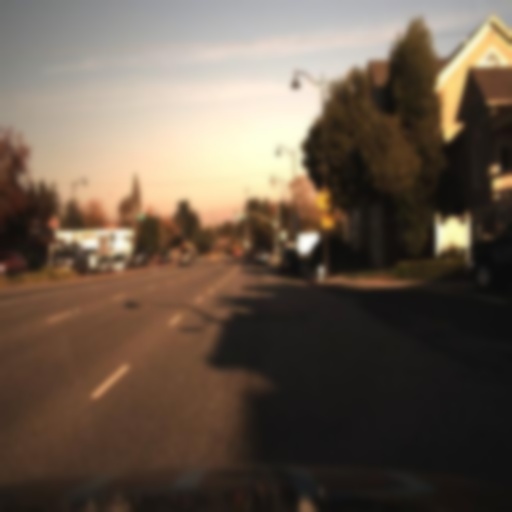}
        \caption{Blurred}
        \label{fig:three-e}
    \end{subfigure}
    \begin{subfigure}[b]{0.15\textwidth}
        \centering
        \includegraphics[width=0.8\textwidth]{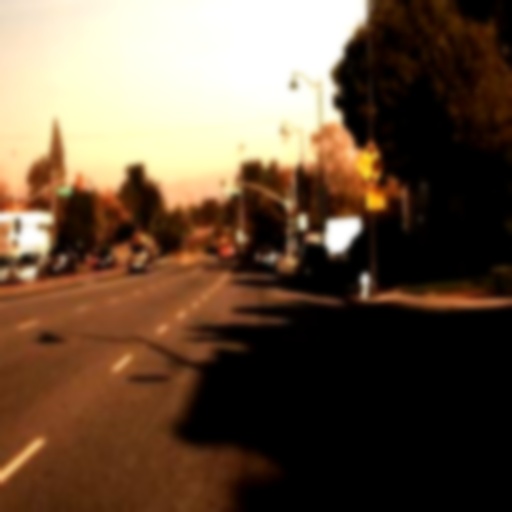}
        \caption{Combined}
        \label{fig:three-f}
    \end{subfigure}

    \caption{Composite perturbations 
    combine the effects of individual perturbations. 
    (\emph{a}) 
    A test with configuration $\boldsymbol{\sigma}=(2,4,2)$,
    indicating 
    \textsf{contrast}, 
    \textsf{scale},
    and \textsf{blur} 
    at severity levels $2$, $4$ and $2$, respectively.
    (\emph{b}) 
    Original input image.
    (\emph{c}--\emph{e}) 
    Effects of exactly one perturbation at the specified severity level.
    (\emph{f}) 
    Combined effect of the composite perturbation.}
    \label{fig:three}
\end{figure*}

The main idea behind \sys is to 
prioritise tests 
to generate a relevant 
approximate model tomography---we term 
this \emph{partial model tomography}.
Our three key assumptions are: (\emph{i}) input data 
can undergo atomic semantic 
perturbations (e.g., scale an image, 
blur it, \emph{or}
add contrast to it; see Figures~\ref{fig:three-b}--\ref{fig:three-e});
(\emph{ii})
each perturbation has discrete severity levels 
(e.g., a scale from $0$ to $5$, where $0$ is negligible and $5$ the most severe level);
and, finally, (\emph{iii}) multiple perturbations 
can be performed in combination 
(e.g., transform an image by 
applying scale, blur \emph{and}
contrast; see Figures~\ref{fig:three-a} and \ref{fig:three-e}).

\emph{Full model tomography} requires an exponential number of tests stemming from all possible combinations of perturbations and their severity levels.  
When the order of perturbation application does not matter, each unique combination of severity levels defines a test. If there are $n$ perturbations and each perturbation $p_j$ has $|S_j|$ severity levels, then there are $\prod_{j=1}^n |S_j|$ 
tests. 
If perturbation order matters, this number grows by a factor of $k!$ for each test with $k$ perturbations at non-zero severity. 
For instance, given three perturbations 
with six severity levels each, as in \autoref{fig:three-a}, 
there are $6^3=216$ unordered and $1 + 15 + 2!\cdot 75+ 3!\cdot 125 = 916$ ordered tests.

To address this combinatorial bottleneck, we propose predicting the robustness score of higher-order tests from \emph{low-order} tests. 
A \emph{$k$-order test} 
involves 
exactly $k$ distinct perturbations with non-zero severity. 
For example, using 
\autoref{fig:three-a} and ignoring 
order,
there is 
one zero-order test
(scale, contrast and blur at severity $0$);
$15$ first-order tests 
with 
exactly one perturbation at non-zero severity (e.g., $\mathsf{scale} > 0$);
$75$ second-order tests 
with
exactly two perturbations at non-zero severity level;
and
$125$ third-order tests 
where all three perturbations are applied 
with non-zero severity.

Each test provides robustness information $\mathrm{r}_{\sigma_1,\ldots,\sigma_k}$,
representing the probability 
that
the model correctly 
classifies
an input when each perturbation $p_j$ is applied at severity level~$\sigma_j$. If perturbations act independently, 
first-order robustness values can be used to predict higher-order ones. For example, in the 3-dimensional perturbation space of \autoref{fig:three-a}, suppose 
$\mathrm{r}_{2,0,0}=0.9$,  $\mathrm{r}_{0,4,0}=0.8$, and $\mathrm{r}_{0,0,2}=0.7$. Under 
independence, 
we 
expect $\mathrm{r}_{2,4,0} = 0.9 \times 0.8 = 0.72$ and $\mathrm{r}_{2,4,2} = 0.9 \times 0.8 \times 0.7 = 0.504$. In other words, if perturbations 
behave independently, partial tomography based solely on first-order tests would 
suffice
to reconstruct the full model tomography.

Unfortunately, we do not 
know \emph{a priori}
whether perturbations act independently or not. However, we can empirically test for independence by comparing first-order tomography results against second-order ones. If the observed second-order values align with predictions derived from first-order data---within a statistical margin---we can treat the corresponding tests as independent and assume this independence holds for higher-order combinations as well.
Returning to 
\autoref{fig:three-a}, 
we could execute the $90$ first- and second-order tests (15 and 75, respectively)
and 
compare the $75$ second-order results against predictions derived from the $15$ first-order ones. If they behave independently, then we could predict the remaining $125$ third-order tests without executing them. 
Independence testing can be achieved with a standard $\chi^2$-test.

\begin{proposition}
\label{prop:chi2}
\normalfont
For a fixed perturbation configuration $\boldsymbol{\sigma} = (\sigma_1, \ldots, \sigma_k)$, let $\mathrm{r}_{\boldsymbol{\sigma}}$ denote the true success probability of a test at severity levels $\sigma_1,\ldots,\sigma_k$, and let $\hat{\mathrm{r}}_{\boldsymbol{\sigma}}$ denote the empirical success rate over $N$ samples.  
Under the assumption that perturbations act independently, the statistic
\[
\chi^2 := \sum_{\boldsymbol{\sigma} \in \mathcal{S}} \frac{( \hat{\mathrm{r}}_{\boldsymbol{\sigma}} - \mathrm{r}_{\boldsymbol{\sigma}} )^2}{ \mathrm{r}_{\boldsymbol{\sigma}} (1 - \mathrm{r}_{\boldsymbol{\sigma}}) / N }
\]
is approximately $\chi^2$-distributed with $|\mathcal{S}|$ degrees of freedom, where $\mathcal{S}$ denotes the set of evaluated severity configurations. 
\end{proposition}
\begin{proof}
For any fixed perturbation configuration $\boldsymbol{\sigma} = (\sigma_1, \ldots, \sigma_k)$, the outcome of a test over $N$ samples follows a binomial distribution:
\[
X_{\boldsymbol{\sigma}} \sim \mathrm{Binomial}(N, \mathrm{r}_{\boldsymbol{\sigma}}),
\quad \text{and} \quad \hat{\mathrm{r}}_{\boldsymbol{\sigma}} := \frac{X_{\boldsymbol{\sigma}}}{N}.
\]
When $N$ is sufficiently large, the binomial distribution is well-approximated by a normal distribution
\[
\hat{\mathrm{r}}_{\boldsymbol{\sigma}} \approx \mathcal{N}\left(\mathrm{r}_{\boldsymbol{\sigma}}, \frac{\mathrm{r}_{\boldsymbol{\sigma}}(1 - \mathrm{r}_{\boldsymbol{\sigma}})}{N} \right)
\]
and we can define the standardised residual
\[
Z_{\boldsymbol{\sigma}} := \frac{ \hat{\mathrm{r}}_{\boldsymbol{\sigma}} - \mathrm{r}_{\boldsymbol{\sigma}} }{ \sqrt{ \mathrm{r}_{\boldsymbol{\sigma}} (1 - \mathrm{r}_{\boldsymbol{\sigma}}) / N } }.
\]
Under the null hypothesis of independence (i.e., that $\mathrm{r}_{\boldsymbol{\sigma}}$ is accurately predicted from lower-order data), each $Z_{\boldsymbol{\sigma}}$ is approximately standard normal. Therefore, the sum:
\[
\chi^2 := \sum_{\boldsymbol{\sigma} \in \mathcal{S}} Z_{\boldsymbol{\sigma}}^2
\]
is approximately $\chi^2$-distributed with $|\mathcal{S}|$ degrees of freedom.
To ensure the normal approximation is valid, the Berry–Esseen theorem implies that $N$ should satisfy:
\[
N > 9 \cdot \max\left\{ \frac{1 - \mathrm{r}_{\boldsymbol{\sigma}}}{\mathrm{r}_{\boldsymbol{\sigma}}}, \frac{ \mathrm{r}_{\boldsymbol{\sigma}} }{1 - \mathrm{r}_{\boldsymbol{\sigma}}} \right\}.
\]
Under this condition, the $\chi^2$-test provides a statistically justified method to evaluate the independence assumption.
\end{proof}
Once dependent and independent perturbations have been identified, the number of required higher-order tests can be significantly reduced. In the ideal case where all first-order perturbations are independent, 
all third-order tests---125 in our 
example, 58\% of the total---can be estimated 
from first-order 
tests 
rather than executed.

If only some perturbations are independent, then higher-order tests can be checked against mutual independence of all constituents. 
For example, if 
\textsf{blur} and \textsf{scale}, 
\textsf{scale} and \textsf{contrast},
and 
\textsf{contrast} and \textsf{blur} are 
all found pairwise independent, 
then the triplet
\textsf{blur-scale-contrast} 
is likely independent as well under realistic scenarios. 
Thus, any higher order test 
that only depends on 
pairwise independent perturbations
can be directly estimated as the product of first-order tests. 
A higher-order test that has some independent and some dependent perturbations can similarly be estimated 
by multiplying the independent first-order tests 
with a set of remaining dependent tests which are lower order, 
thus reducing the computational cost. 

More generally, however, lower-order results can be used to train predictive models that \emph{account for dependencies} when estimating higher-order outcomes. This approach dramatically reduces the computational cost of tomography while maintaining accuracy in user-specified test environments. We demonstrate this in \S\ref{sec:random-forest}, where we train a random forest on first- and second-order test results to predict third-order outcomes and beyond.

\section{TestifAI}

We describe \sys, a test framework for deep learning models that performs partial tomography. We begin with an overview of the framework's architecture, followed by a description of how users interact with it (\S\ref{sec:user-interactions}), how lower-order tests are executed efficiently (\S\ref{sec:sample-efficiency}), and how higher-order tests are predicted using a learned model (\S\ref{sec:random-forest}).

A test session with \sys comprises four stages:
(\emph{i}) users 
specify
a test environment for their model by 
selecting
perturbations and their severity levels that best characterise the application domain; (\emph{ii}) \sys 
automatically 
generates and executes all first- and second-order tests within that test environment; (\emph{iii}) based on these results, \sys trains an auxiliary predictive model---in our case, a random forest---to approximate the model's full tomography space; (\emph{iv}) users can then query \sys to estimate robustness 
over
regions of 
the
test environment.

For illustrative purposes, 
we
focus 
on \textbf{3D tomography}, where the test environment consists of three 
perturbations. In \S\ref{sec:result2}, we further explore the generality of our approach to 4D
tomography.

\subsection{Specifying tests}
\label{sec:user-interactions}

Users can customise testing both before and after tomography.
While \sys supports a wide range of perturbations and severity levels, users may choose to \emph{setup} a customised test environment for their model---filtering specific perturbations, severity levels, or both. For example, the use of the \textsf{zoom} perturbation might be constrained by a camera's focal range.
After \sys estimates the full tomography space, users may interactively \emph{query} specific regions to assess model robustness or refine their setup---for instance, to determine whether a given severity level meaningfully impacts model robustness or not.

\subsubsection{Test Setup}

Users begin with a trained
model \(\mathcal{M}\) 
(e.g., \textrm{ResNet-32}~\cite{croce2021robustbench})
and an evaluation data set \(\mathcal{D}\) (e.g., the 10,000 test images of CIFAR-10~\cite{CIFAR10}).
Having assessed standard accuracy, they now aim to evaluate the model's robustness.

During setup, \sys
enables users to select $k$ perturbations 
$p_{1},\dots,p_{k}$
from a predefined set, 
along with their associated severity levels,
to define a test environment.
 For example, 
 CIFAR-10-C~\cite{hendrycks2019cifar10c} includes $15$ 
 well-defined perturbations, 
 each with \textit{6} severity levels. 
 Alternatively, users may provide 
 custom perturbation functions 
 tailored to their application 
 or data domain.
 
\sys defines the full tomography space 
as the set of all possible severity configurations 
across the selected perturbations:
\[
  \Theta 
  = S_{1} \times \cdots \times S_{k},
\]
where $S_{j}$ is the set of selected 
severity levels for perturbation $p_{j}$.
Each element \(\boldsymbol{\sigma} = (\sigma_1,\dots, \sigma_k) \in \Theta\) represents a single \emph{test}---a complete assignment of severity levels across all $k$ perturbations that must be either executed or predicted.

\subsubsection{Querying $\Theta$}

After \sys 
learns robustness estimates across the space
$\Theta$, users can express a robustness query as a Boolean expression over perturbation severities:
\[
  Q \;::=\;
    (p_j~\mathsf{op}~\sigma_j)
    \;\big|\;
    Q \land Q
    \;\big|\;
    Q \lor Q,
\]
where 
\((p_j~\mathsf{op}~\sigma_j)\) is an 
\emph{atomic constraint} 
defining the considered range of values for perturbation \(p_j\), 
where $\mathsf{op}$ is a comparison operator \(\mathsf{op} \in \{<,\leqslant,=,\geqslant,>\}\) and 
\(\sigma_j \in S_j\) a severity level
(e.g., ``\textsf{blur}~$>2$'').
The compound expressions
\(Q \land Q\) 
and 
\(Q \lor Q\) 
denote the logical conjunction 
and disjunction of two subqueries, 
respectively.
For example,
\(
  Q = (\textsf{zoom}>2)\;\lor\;(\textsf{brightness}=5\land\textsf{blur}=1)
\)
selects all tests where the zoom severity exceeds $2$, or where both brightness is set to $5$ and blur is set to $1$.

A query $Q$ represents a subset of the tomography space $\Theta$. We denote by 
$\llbracket Q \rrbracket$ the set of tests
$\boldsymbol{\sigma} \in \Theta$
that satisfy the query expression~$Q$.
Since the robustness $\mathrm{r}_{\boldsymbol{\sigma}}$ is 
known---either measured or predicted---for every test, \sys computes the aggregate robustness 
of $Q$ using \autoref{eq:global-robustness}:
$\mathrm{R}(Q) := \mathrm{R}(\llbracket Q \rrbracket)$.

\begin{algorithm}[t]
\small
\caption{Constructing the training dataset for partial tomography by selecting and evaluating tests from $\Theta_{\leqslant 2}$, the space of first- and second-order perturbation combinations.}
\label{alg:testifai_full}
\begin{algorithmic}[1]
\State \textbf{Inputs:}
Trained model $\mathcal{M}$; dataset $\mathcal{D}$; tests $\Theta_{\leq 2}$; batch size $b$; threshold $\delta$; window size $w$
\State \textbf{Output:} Training set $\mathcal{T}$
\State \(\mathcal{T} \gets \emptyset \)
\ForAll{$\boldsymbol{\sigma} \in \Theta_{\leq 2}$}
  \State Partition $\mathcal{D}$ into batches $\mathbb{B} = \{B_1, B_2, \dots, B_n\}$ of size $b$ \label{algline:batch}
  \State \( \mathfrak{h}_{\boldsymbol{\sigma}} \gets [\,0\,]^n \) \Comment{Reset history}
  \State $i \gets 1$
  \ForAll{$B \in \mathbb{B}$}
    \State $\tilde{B} \gets \{\,(\pi_{\boldsymbol{\sigma}}(x), y) \mid (x, y) \in B \,\}$ \Comment{ Apply perturbation} \label{algline:transform}
    \State $\mathcal{T} \gets \mathcal{T} \cup \{\, \big(\boldsymbol{\sigma}, \ \mathbf{1}[\mathcal{M}(\tilde{x}) = y]\big) \mid (\tilde{x}, y) \in \tilde{B} \,\}$ \Comment{Coll. tr. data} \label{algline:train}
    \State $\mathfrak{h}_{\boldsymbol{\sigma}}^{(i)} \gets \frac{1}{b} \displaystyle\sum_{(\tilde{x}, y) \in \tilde{B}} \mathbf{1}\bigl[\mathcal{M}(\tilde{x}) = y\bigr]$ \Comment{Store partial result} \label{algline:hist}
    \If{ $i \geqslant w \text{ and } \displaystyle\max_{j=i-w+1}^{i} \left|\,\, \mathfrak{h}_{\boldsymbol{\sigma}}^{(j)} - \mathfrak{h}_{\boldsymbol{\sigma}}^{(j-1)} \,\,\right| < \delta$ } \label{algline:stop}
    \State \textbf{break}
    \EndIf
    \State $i \gets i + 1$
  \EndFor
\EndFor
\end{algorithmic}
\end{algorithm}

\subsection{Executing tests}
\label{sec:sample-efficiency}

\sys evaluates all first- and second-order tests and uses these results to predict robustness of third-order tests and beyond. 
We partition the full tomography space $\Theta$
by the number of active perturbations.
We define the \emph{$t$-th order subset $\Theta_t$}, for $t \leqslant k$, as:
\[
\Theta_t
  = \left\{
     \boldsymbol{\sigma} = (\sigma_1,\dots,\sigma_k) \in \Theta\phantom{^i}
     \;\middle|\;
     \text{$\boldsymbol{\sigma}$ has exactly $t$ non-zero $\sigma_i$-values}
     \right\}.
\]
In other words, $\Theta_t$ 
consists of all tests where exactly 
$t$ perturbations 
are applied with non-zero severity. 
For example, in 3D tomography ($k=3$), 
the tomography space is partitioned into
$\Theta_0$, $\Theta_1$, $\Theta_2$, and $\Theta_3$,
corresponding to zeroth-, first-, second-, and 
third-order tests, respectively. $\Theta_0$ only contains the ``no-perturbations at all'' case, i.e., it is merely the base model accuracy. 

\sys evaluates the robustness 
of all tests in the set \(\Theta_{\leqslant2} = \Theta_0 \cup \Theta_1 \cup \Theta_2\)
by applying each perturbation 
configuration \(\boldsymbol{\sigma} \in \Theta_{\leqslant 2}\)
to inputs from the dataset \(\mathcal{D}\), 
running the model to infer 
the label of each perturbed input, 
and recording success or 
failure based on label correctness (Algorithm~\ref{alg:testifai_full}).
{With $k$ perturbations and $m$ severity levels per perturbation,
$|\Theta_1| = km$ and $|\Theta_2| = \binom{k}{2} m^2$. 
So the total number of tests executed is $O(k^2 m^2)$,
eliminating the 
exponential $O(m^k)$ cost 
of full tomography.}

However, inferring the label of every perturbed input 
is often unnecessary to compute a good estimation of $\mathrm{r}_{\boldsymbol{\sigma}}$.
\sys employs an early-stopping strategy to avoid superfluous model inferences 
(see Algorithm~\ref{alg:testifai_full}). 
The idea is to estimate $\mathrm{r}_{\boldsymbol{\sigma}}$ incrementally.
First, \sys partitions the dataset into small
batches of size \(b\) ($\ell$.~\ref{algline:batch}).
It then iteratively computes and stores
a per-batch robustness estimate ($\ell$.~\autoref{algline:hist}).
\sys will assess convergence using a window of the last 
$w$ partial estimates.
Computation 
stops when the variation in the window 
falls below a predefined threshold $\delta$
($\ell$.~\autoref{algline:stop}).
We empirically found that \(b = 100\), \(\delta = 0.005\) and \(w = 3\) 
gives a good, 
unbiased estimate of $\mathrm{r}_{\boldsymbol{\sigma}}$ 
in our experiments.
{The time cost of 
each test comprises 
the cost of transforming 
data samples ($\ell$.~\ref{algline:transform}) and 
the cost of performing 
model inference on them ($\ell$.~\ref{algline:train}). 
Early
stopping reduces both components 
by limiting the number of samples processed.
We discuss the 
computational savings
and
the relative
contributions of transformation 
and inference time in \S\ref{sec:result2}.}

All binary prediction outcomes observed prior to early stopping  
are stored in a set \(\mathcal{T}\) ($\ell$.~\autoref{algline:train}),
which is then used for training our predictive model.

\subsection{Predicting tests}
\label{sec:random-forest}

\sys learns to predict the robustness 
of higher-order tests in \(\Theta_{\geqslant3}\)
based on empirical observations 
from the lower-order tests
\(\Theta_{\leqslant2}\).
Given a perturbation configuration \(\boldsymbol{\sigma}\in\Theta\), 
our predictive model
\(\mathcal{O}\) returns a predicted robustness score \(\hat{\mathrm{r}}_{\boldsymbol{\sigma}} = \mathcal{O}(\boldsymbol{\sigma})\).

\paragraph{Training}
Algorithm~\autoref{alg:testifai_full} returns a 
training data set
\[
\mathcal{T} = \left\{\, \big(\boldsymbol{\sigma},\ \mathbf{1}[\mathcal{M}(\pi_{\boldsymbol{\sigma}}(x)) = y]\big) \;\middle|\; 
\boldsymbol{\sigma} \in \Theta_{\leqslant 2},\; (x, y) \in \mathcal{D}_{\boldsymbol{\sigma}} 
\,\right\},
\]
where \(\mathcal{D}_{\boldsymbol{\sigma}} \subseteq \mathcal{D}\) 
is the subset of the inputs---possibly partial, 
since we employ early 
stopping in Algorithm~\ref{alg:testifai_full}---used 
to estimate 
the robustness of the model 
under perturbation configuration \(\boldsymbol{\sigma}\).
The size of the training set \(\vert\mathcal{T}\vert = \sum_{\boldsymbol{\sigma} \in \Theta_{\leqslant 2}} |\mathcal{D}_{\boldsymbol{\sigma}}|\) is the total number of perturbed inputs evaluated across all configurations.
Each element of \(\mathcal{T}\) pairs a configuration \(\boldsymbol{\sigma}\) 
with a binary outcome indicating whether the model correctly classified a given perturbed input or not. 
We treat each element of $\mathcal{T}$ 
as a training example consisting 
of a perturbation configuration 
and its corresponding binary outcome.
By projecting \(\mathcal{T}\), we construct 
a feature matrix \(X_{\text{\tiny $\mathcal{O}$}} \in \mathbb{R}^{\vert\mathcal{T}\vert \times k}\)
containing all configurations and a label vector 
\(\mathbf{y}_{\text{\tiny $\mathcal{O}$}} \in \{0,1\}^{\vert\mathcal{T}\vert}\) containing all outcomes. These are then used to train our model.

Our model \(\mathcal{O}\) is a {random forest classifier}. After tuning, we selected the following configuration:
\begin{enumerate}
\item The model consists of $100$ trees, 
balancing computational and statistical performance.
\item Bootstrap sampling is disabled, 
allowing each tree to train on the full dataset.
\item There is
no restriction on the number of 
features considered at each split, enabling trees to 
explore the full feature space and capture richer interactions 
among perturbation types.
\item It uses the log-loss splitting criterion, optimizing for splits that reduce the cross-entropy between predicted and true labels. This encourages probability estimates
of robustness that are more reliable and easier to interpret.
\end{enumerate}
This configuration was selected based 
on the lowest mean squared error (MSE) 
observed on a held-out validation set 
of actual \(\Theta_3\) test results. 
We further evaluate 
our model's generalization performance 
on \(\Theta_3\) and \(\Theta_4\) 
in \S\ref{sec:result2}.

An analogous surrogate model can be trained to estimate perturbation validity (e.g., estimate KID~\cite{BinkowskiSAG18} and BERTScore~\cite{zhang2020bertscore} for perturbed images and text, respectively) from low-order observations, enabling users to exclude low-quality regions of \(\Theta\)  (see \S\ref{sec:validity}).

\paragraph{Why random forests?}  
Random forests are non-parametric
ensemble methods that approximate structured 
conditional distributions without explicit 
structure specification---unlike 
Bayesian networks or 
factor graphs. 
{Also, in our preliminary experiments, random forests had the best
sample efficiency among other architectures (gradient-boosted trees,
multilayer perceptrons, 
and Bayesian networks) that achieved 
comparable accuracy. This was an important factor since 
we train \(\mathcal{O}\) on a limited number 
of empirical robustness measurements.
}
Note, however, that the choice of the best architecture for \(\mathcal{O}\) is not a focus of this work.

\section{Evaluation}
\label{sec:eval}

We structure our evaluation around three key research questions:
(\emph{i}) How accurate are \sys’s robustness predictions? (\S\ref{sec:result1});
(\emph{ii}) Is partial tomography an effective strategy 
for approximating the robustness space $\Theta$? (\S\ref{sec:result2}); and
(\emph{iii}) How sample-efficient is \sys in estimating robustness? (\S\ref{sec:result3}).

\subsection{Experimental setup}

We implemented \sys 
in 
Python~3.9, 
using the \textsf{scikit-learn} 
library to train our random forest model, 
and \textsf{sympy} 
to parse and evaluate Boolean query expressions. 
Experiments were 
conducted on a 
high-performance GPU cluster 
at the Massachusetts 
Green High Performance Computing Center (MGHPCC), 
using an NVIDIA Tesla T4 GPU with CUDA 12.3.

\begin{table}[t]
  \centering
  \small
  \renewcommand{\arraystretch}{0.85}
  \setlength{\tabcolsep}{3pt}
  \caption{Summary of the benchmarks and evaluated models.}
  \label{tab:datasets}
  \begin{tabular}{@{}clrp{3.1cm}lc@{}}
    \toprule
    \textbf{Task} & \textbf{Dataset} & \textbf{Size} & \textbf{Perturbations} & \textbf{Model} & \textbf{Acc.} \\
    \midrule
    \mnistSymb & MNIST    & 10,000 & brightness, zoom, motion-blur, shear                & LeNet-5    & 98.4\% \\
    \cifarSymb & CIFAR-10 & 10,000 & shot-noise, brightness, jpeg-compression,  contrast & WRN-28-10 & 94.7\% \\
    \sdcarSymb & Roboflow &  1,000 & translate, scale, contrast, brightness              & YOLOv11    & 82.2\% \\
    \tsignSymb & GTSRB    & 12,630 & darken, codec-error, gaus\-sian-blur, exposure        & CNN-SE     & 97.6\% \\
    \quoraSymb & QQP      &  1,000 & synonym, typos, contraction, punctuation            & RoBERTa    & 91.2\% \\
    \bottomrule
  \end{tabular}
\end{table}

\subsubsection{Benchmarks}
\label{sec:benchmarks}

We evaluate \sys on five benchmarks: 
four
vision and one 
language classification tasks.
For each task, 
we found a 
publicly available 
pre-trained classification
model and its associated data set. 
Table~\ref{tab:datasets} 
summarises our benchmarks: the dataset
name, size and the considered perturbations, 
as well as the model name and its
classification accuracy. 

\emph{Hand-written digit recognition ({\footnotesize \mnistSymb})} 
is a 
classic computer vision classification task.
We test the robustness of the
\textsf{LeNet-5} model~\cite{yuan2022deepboundary}
on perturbed  
grayscale 
images of hand-written digits from
the \textsf{MNIST} test dataset~\cite{gitDislMnistDatasets}.
We assess model robustness to \textsf{brightness}, \textsf{zoom}, \textsf{motion blur} 
and \textsf{shear}---four common digit‐image corruptions~\cite{mu2019mnistc,brendel2018ai, mohsenzadegan2021deep,kochkorova2025data,shorten2019survey} that mimic lighting changes, scale variations, camera motion, and geometric distortions, respectively.
Perturbations were implemented using the MNIST-C library~\cite{google_research_mnistc_corruptions}.

\emph{Image classification ({\footnotesize \cifarSymb})}
is another classic vision task. 
We test the robustness of \textsf{WideRes\-Net-28-10}~\cite{zagoruyko2016wide,croce2021robustbench} on perturbed images of the CIFAR-10 test dataset~\cite{CIFAR10}.
We assess model robustness to \textsf{shot noise}, \textsf{brightness}, \textsf{jpeg compression}, and \textsf{contrast}~\cite{hendrycks2019benchmarking, hendrycks2020augmix, hendrycks2022pixmix,calian2021defending} 
that simulate sensor imperfections, 
illumination changes, 
compression artefacts, 
and visibility variations, respectively. 
Perturbations were implemented using RobustBench~\cite{croce2021robustbench}.

\emph{Object detection in self-driving Cars ({\footnotesize \sdcarSymb})} is part of the Udacity challenge---the task is to detect objects in urban driving images~\cite{Roboflow}. We select $1000$ images to test the robustness of the \textsf{YOLOv11s} model~\cite{pyresearch_yolo11, ultralytics2025} against \textsf{scale}, \textsf{contrast}, \textsf{translation}, and \textsf{brightness}, four common driving scene corruptions~\cite{chen2024end, chandrasekaran2021combinatorial, tian2018deeptest} that mimic distance variations, lighting conditions, camera movements, and illumination changes, respectively.
Perturbations were implemented using DeepTest~\cite{tian2018deeptest}.
We selected five 
(out of the 
ten) 
severity levels, setting
scale $(s_x,s_y)\in \{(1.5,1.5),\allowbreak\, (2.6,2.6),\allowbreak\, (3.7,3.7),\allowbreak\, (4.8,4.8),\allowbreak\, (6.0,6.0)\}$, contrast $\alpha\in \{1.2,\allowbreak\,1.6,\allowbreak\,2.1\allowbreak,\,2.5,\allowbreak\,3.0\}$, translation $(t_x,t_y)\in \{(20,20),\allowbreak\, (40,40)\allowbreak,\, (60,60),\allowbreak\, (80,80),\allowbreak\, (100,100)\}$, and brightness $\beta\in \{20,\allowbreak\,40,\allowbreak\,60,\allowbreak\,80,\allowbreak\,100\}$.

\emph{Traffic sign recognition ({\footnotesize \tsignSymb})} is 
an essential 
task for autonomous driving systems. We test the robustness of the \textsf{CNN-SE} model~\cite{Neonithinar_GTRSB} on the German Traffic Sign Recognition Benchmark (GTSRB)~\cite{stallkamp2011german}. The dataset contains $12630$ test images 
of $43$ signs. We 
test
\textsf{darkening}, \textsf{codec-error}, \textsf{gaussian blur}, and \textsf{exposure}---four perturbations that simulate  
challenges in traffic sign perception: nighttime or shadowed viewing conditions, video transmission artifacts, imperfect camera focus, and overexposed imaging, respectively~\cite{temel2018traffic,yan2023traffic,aldoski2025traffic}.
Perturbations were implemented using CURE-TSD~\cite{olivesgatech_CURETSD}.

\begin{figure*}[t]
    \centering
    \begin{subfigure}[b]{0.196\textwidth}
        \includegraphics[width=\textwidth]{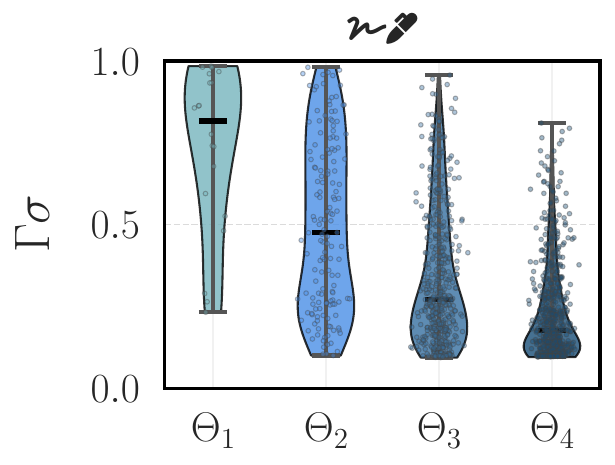}
    \end{subfigure}
    \begin{subfigure}[b]{0.196\textwidth}
        \includegraphics[width=\textwidth]{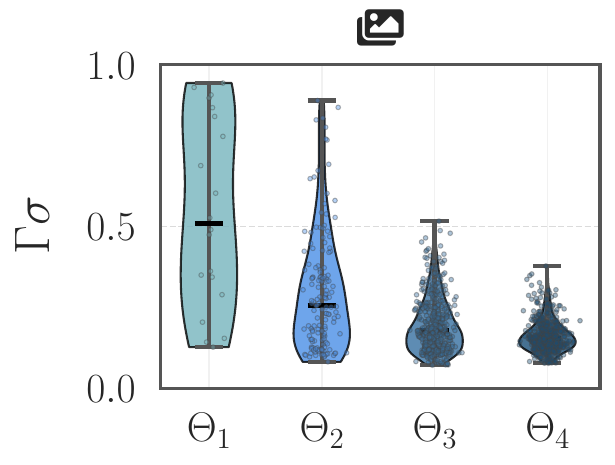}
    \end{subfigure}
    \begin{subfigure}[b]{0.196\textwidth}
        \includegraphics[width=\textwidth]{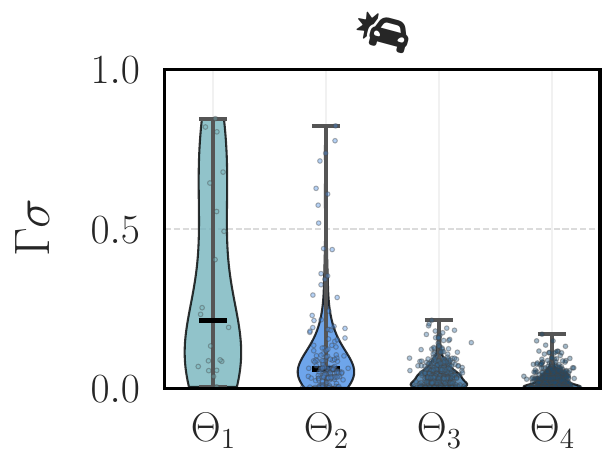}
    \end{subfigure}
    \begin{subfigure}[b]{0.196\textwidth}
        \includegraphics[width=\textwidth,height=2.7cm]{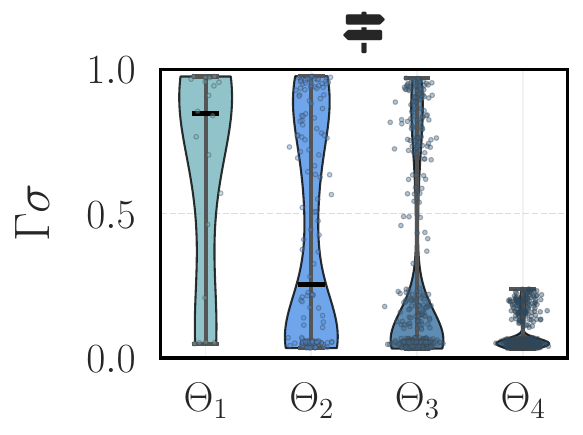}
    \end{subfigure}
    \begin{subfigure}[b]{0.196\textwidth}
        \includegraphics[width=\textwidth]{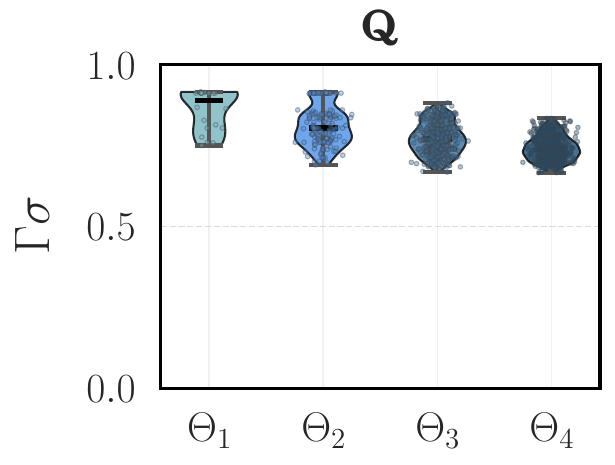}
    \end{subfigure}
    \caption{Distribution of ground-truth 
    {robustness} scores 
    $\mathrm{r}_{{\boldsymbol{\sigma}}}$ 
    across
    $\Theta_1$–$\Theta_4$
    for each benchmark. 
    Higher orders show wider,
    downward-shifted distributions,
    indicating
    increased accuracy degradation and robustness variability.
    }
    \label{fig:dist1234}
\end{figure*}

\begin{figure}[t]
    \centering
    \begin{minipage}[b]{.46\columnwidth}
        \centering
        \includegraphics[width=0.91\columnwidth]{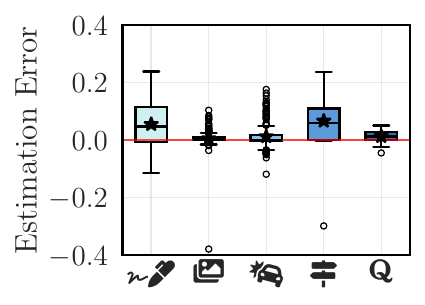}
        \caption{\({\mathrm{r}}_{\boldsymbol{\sigma}}\) estimation errors for 3D tomography.}
         \label{fig:error-third-order-stability}
    \end{minipage}
    \hspace{1em}
    \begin{minipage}[b]{0.46\columnwidth}
        \centering
        \includegraphics[width=1\columnwidth]{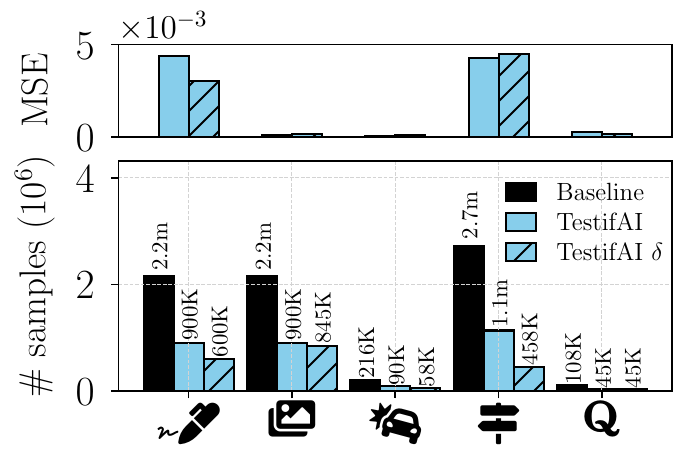}
        \caption{Sample efficiency for 3D tomography.}
        \label{fig:Q13_mse_bar_fullbox_labels}
    \end{minipage}
\end{figure}

\emph{Quora question-answering ({\footnotesize \quoraSymb})} is a semantic similarity task
for natural language understanding~\cite{wang2018glue}. We test the robustness of the $\textsf{RoBERTa}_{\textsf{base}}$~\cite{roberta-base-qqp} model. For our test dataset, we choose $1000$ question pairs from the Quora Question Pairs (QQP) dataset, each pair
having a binary label indicating semantic equivalence. 
We assess model robustness to \textsf{synonym replacement}, \textsf{typos}, \textsf{contractions}, and \textsf{punctuation}~\cite{ribeiro2020beyond,gao2023heros,xia2021using,wang2021textflint}---four perturbations that preserve meaning while introducing lexical, orthographic, stylistic, and structural variations, respectively.
Perturbations were implemented using
TextAttack~\cite{morris2020textattack}.
Each severity level---1 through 5--directly corresponds to the number of edits applied to a sentence: at level 1 we make one edit, level 2 two edits, and so on, up to level 5.

\subsubsection{Ground Truth}
\label{subsec:performance_metrics}

We evaluate 
the accuracy of \sys 
by comparing its predictions 
against full model tomography. 
This baseline exhaustively computes 
the true robustness
score $\mathrm{r}_{\boldsymbol{\sigma}}$ 
or every test $\boldsymbol{\sigma}$ 
in the selected perturbation space. 
In other words, for each benchmark, 
we apply every perturbation configuration 
to every input in the 
evaluation dataset to obtain 
exact 
values.
Figure~\ref{fig:dist1234} 
presents the ground-truth distributions 
of $\mathrm{r}_{\boldsymbol{\sigma}}$.
Partial tomography relies 
on the statistical patterns 
of the \(\Theta_1\) and \(\Theta_2\)
distributions to predict
the \(\Theta_3\) and \(\Theta_4\)
ones.

\begin{figure}[t]
\centering

\begin{minipage}{0.13\textwidth}
  {\centering
    \includegraphics[width=\linewidth]{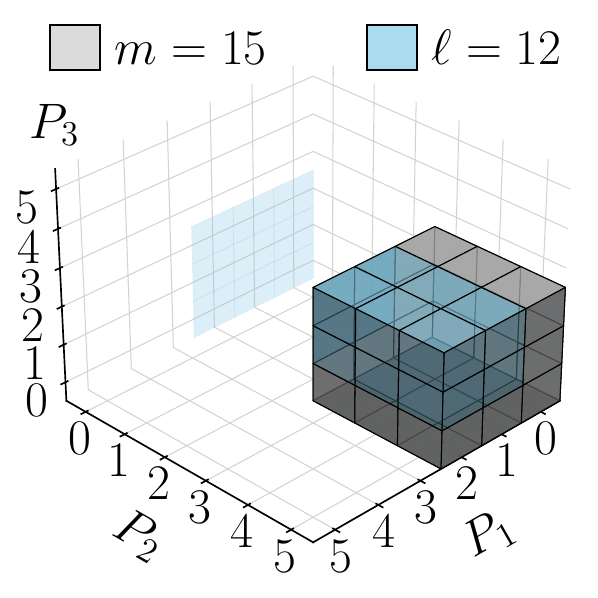}\par
  }
  \vspace{-2mm}
  \makebox[\linewidth][l]{\hspace{-1.5mm}\small $Q_3$}
  \vspace{0.5mm}
  {\centering\scriptsize $P_1 \le 2 \wedge P_2 \ge 3 \wedge P_3 \le 2$\par}
\end{minipage}
\hspace{1.0em}
\begin{minipage}{0.13\textwidth}
  {\centering
    \includegraphics[width=\linewidth]{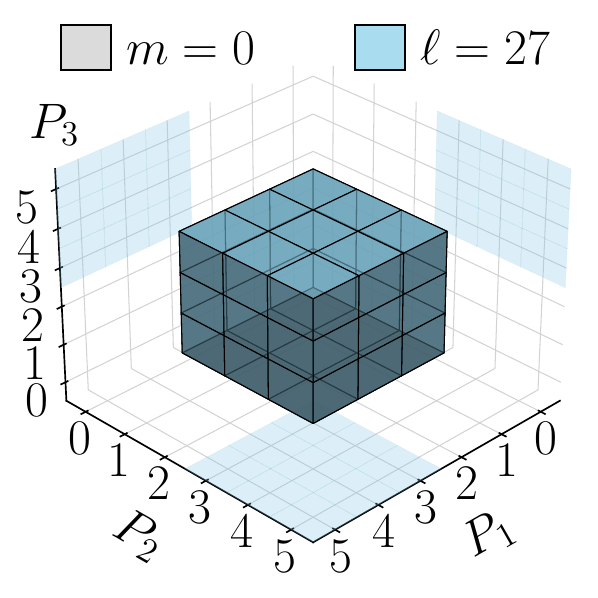}\par
  }
  \vspace{-2mm}
  \makebox[\linewidth][l]{\hspace{-1.5mm}\small $Q_8$}
  \vspace{0.5mm}
  {\centering\scriptsize $P_1 \ge 3 \wedge P_2 \ge 3 \wedge P_3 \ge 3$\par}
\end{minipage}
\hspace{1.0em}
\begin{minipage}{0.13\textwidth}
  {\centering
    \includegraphics[width=\linewidth]{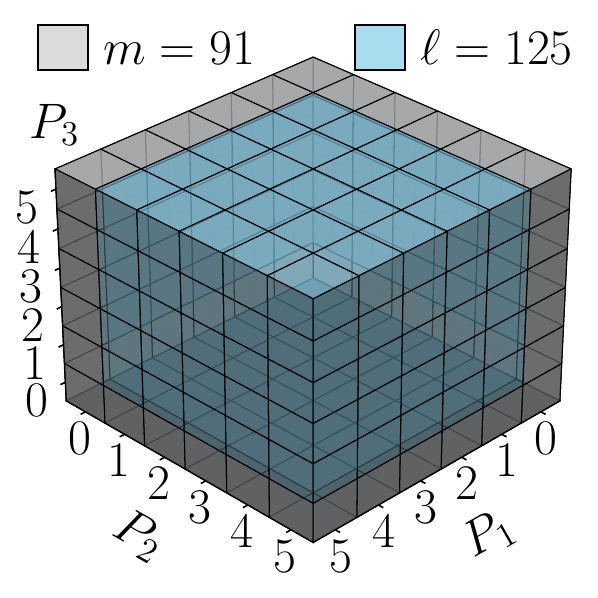}\par
  }
  \vspace{-2mm}
  \makebox[\linewidth][l]{\hspace{-1.5mm}\small $Q_{12}$}
  \vspace{0.5mm}
  {\centering\scriptsize $P_1 \le 5 \wedge P_2 \le 5 \wedge P_3 \le 5$\par}
\end{minipage}

\caption{3D voxel visualisation of $Q_3$, $Q_8$ and $Q_{12}$.}
\label{fig:voxel-grid}
\end{figure}

\begin{table}
    \centering
    \small
    \renewcommand{\arraystretch}{0.85}
    \caption{Summary of the number ($m$) of tests measured for training the partial tomography model and the number ($\ell$) of tests whose robustness was inferred, for each query $Q_i$. }
    \begin{tabular}{c|rrrrrrrrrrrr}
         \toprule
         $i$  &  1 &  2 &  3 &  4 &  5 &  6 &  7 &  8 &  9 & 10 & 11 & 12\\
         \midrule
         $m$  & 19 & 15 & 15 &  9 & 15 &  9 &  9 &  0 &  7 & 37 & 61 & 91\\
         $\ell$  &  8 & 12 & 12 & 18 & 12 & 18 & 18 & 27 &  1 & 27 & 64 & 125\\
         \bottomrule
    \end{tabular}
    \label{tab:queries-stats}
\end{table}

\begin{figure*}[t]
    \centering
        \includegraphics[width=\textwidth]{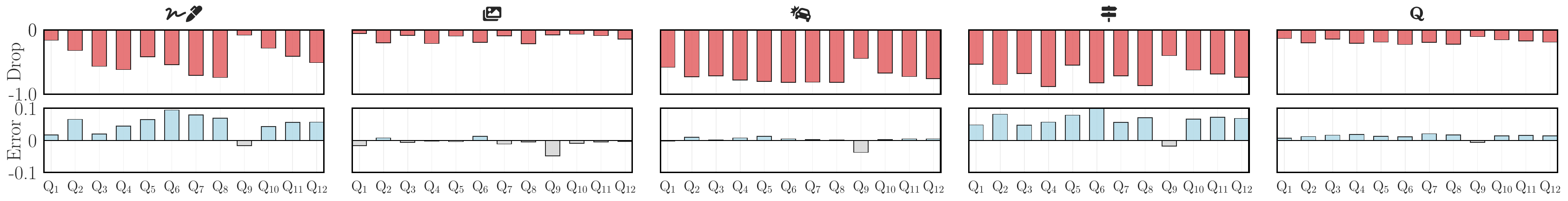}
    \caption{
    Robustness estimation error (error) and 
    predicted drop in $\M$'s accuracy (drop) for 
    aggregate queries $Q_1,\dots,Q_{12}$.  
    For a query $Q$, 
    the error is 
    $\hat{\mathrm{R}}(Q) - \mathrm{R}(Q)$, 
    and the drop 
    is 
    the predicted accuracy of $\M$ minus 
    its
    original accuracy 
    (as per Table~\ref{tab:datasets}).
    }
    \label{fig:query-error-third-order-stability}
\end{figure*}

\subsection{Robustness estimation errors}
\label{sec:result1}

We evaluate the ability 
of our partial tomography 
model---namely, a random forest---to approximate 3D tomography.
For each of our five benchmarks,
we consider the first three perturbations
listed in
Table~\ref{tab:datasets},
each discretised into six severity levels.
This results in
$216$ unique tests per benchmark, 
corresponding to all possible
combinations of three 
perturbations and their severity levels.
We train our random forest
following the setup 
described in \S\S\ref{sec:sample-efficiency} and~\ref{sec:random-forest}.
The model is trained on data
obtained from all 
first- and 
second-order tests (i.e., $\Theta_{\leqslant 2}$),
and is then used to predict robustness scores \(\hat{\mathrm{r}}_{\boldsymbol{\sigma}}\) for all ${\boldsymbol{\sigma}} \in \Theta$, including the 126 
third-order {tests} 
that were not seen during training.

\subsubsection{Per-test errors}

We assess how well 
the model predicts the
robustness score
for each test
by comparing
predicted scores
\(\hat{\mathrm{r}}_{\boldsymbol{\sigma}}\)
with their 
ground truth values 
\({\mathrm{r}}_{\boldsymbol{\sigma}}\).
We define the \emph{robustness estimation error} 
as 
the difference
\(\hat{\mathrm{r}}_{\boldsymbol{\sigma}} - {\mathrm{r}}_{\boldsymbol{\sigma}}\).
A positive error indicates that the model is pessimistic (i.e., it underestimates robustness), while a negative error indicates optimism 
(i.e., overestimation).

Figure~\ref{fig:error-third-order-stability}
shows a box-and-whiskers 
plot of robustness estimation 
errors across all 216 tests (\(\Theta\))
for 3D tomography. 
For each benchmark, the box indicates the median and interquartile range of the errors, while the whiskers
show their full spread.
Across all five benchmarks, most errors cluster tightly around zero, indicating high estimation accuracy. For {\footnotesize \cifarSymb} and {\footnotesize \quoraSymb}, all prediction errors fall within the range $[-0.05, +0.05]$. More broadly, over 90\% of errors across the 216 {tests} lie within $[-0.1, +0.1]$.
Of the 216 {tests}, 90 were used during training, which explains why many exhibit near-zero error---these were directly learned by the model. However, the low error on the remaining 126 held-out tests demonstrates the model’s ability to generalize beyond the training set. For example, the mean prediction error for {\footnotesize \cifarSymb}, {\footnotesize \sdcarSymb}, and {\footnotesize \quoraSymb} is close to zero.

\subsubsection{Aggregate query errors}

We evaluate our model's ability to answer robustness queries
that go beyond individual tests, using a predefined set
of Boolean queries $Q_1$--$Q_{12}$. 
These queries cover the test space \(\Theta\) in complementary ways:
\begin{enumerate*}[label=(\textit{\roman*})]
\item 
$Q_1$--$Q_8$ are \emph{exclusive}---they 
divide  \(\Theta\) 
into mutually non-overlap\-ping
regions. Each exclusive query covers exactly
12.5\% of \(\Theta\) (i.e., $27$ individual tests). 
E.g., $Q_1 = P_1\le2\wedge P_2\le2\wedge P_3\le2$ corresponds to the region with the lowest severity of perturbations, while $Q_8 = P_1\ge3\wedge P_2\ge3\wedge P_3\ge3$ with the highest. 
These queries are designed to 
isolate specific failure modes 
and reveal how robustness varies 
across distinct areas of 
the perturbation space.
\item
$Q_9$--$Q_{12}$ are \emph{inclusive}---each 
covers a progressively larger subset \(\Theta\),
enabling fine-to-coarse analysis.
For example,
$Q_9$ covers only 3.7\% of the space ($8$ tests),
while $Q_{12}$ covers 100\% of it (all 216 tests).
These queries reflect realistic 
scenarios where users may 
wish to assess robustness under broader deployment conditions.
\end{enumerate*}
We visualise the perturbation space covered by some queries in  \autoref{fig:voxel-grid}, and summarise in Table~\ref{tab:queries-stats} the number of measured and of inferred tests for 
$Q_1$--$Q_{12}$.

In \autoref{fig:query-error-third-order-stability}, for each query $Q$, we report the robustness estimation error 
as the difference between the predicted
\(\hat{\mathrm{R}}(Q)\) and the actual \({\mathrm{R}}(Q)\).
By evaluating 
twelve predefined query regions---each 
covering
a distinct portion of the perturbation space---we assess 
whether our random forest trained solely 
on low-order tests 
can accurately estimate robustness 
degradation over increasingly complex subspaces, 
without relying on high-order test data.
\autoref{fig:query-error-third-order-stability} 
shows that 
estimation errors 
remain consistently low across 
all 
queries. $Q_8$ is 
noteworthy 
because it contains only third-order tests, all 
unseen during training;
yet, its maximum error
is just $0.109$, demonstrating
the model's ability to generalize.

\subsection{Effect of partial model tomography}
\label{sec:result2}

\begin{figure*}[t]
    \centering
    \begin{subfigure}[b]{0.195\textwidth}
        \includegraphics[width=\textwidth]{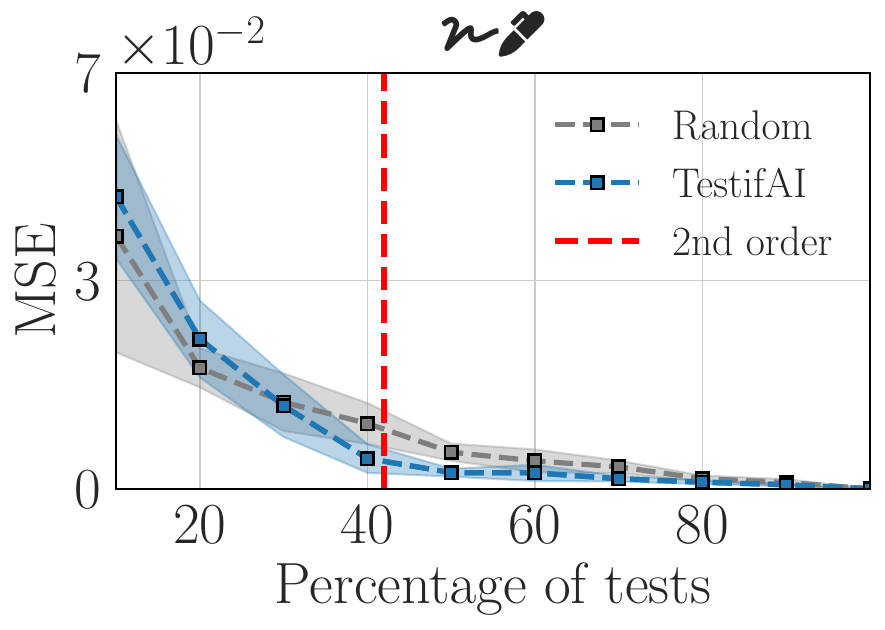}
    \end{subfigure}
    \begin{subfigure}[b]{0.195\textwidth}
        \includegraphics[width=\textwidth]{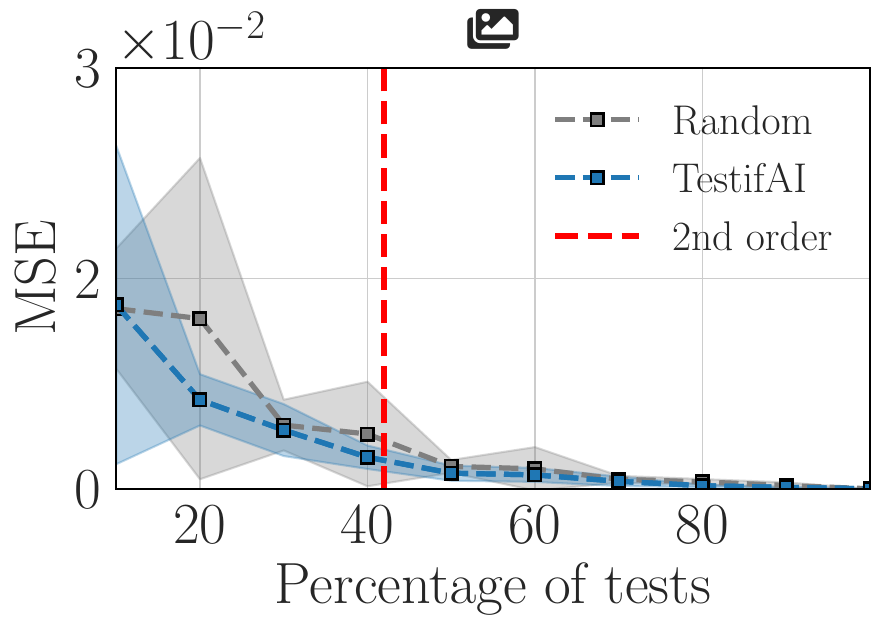}
    \end{subfigure}
    \begin{subfigure}[b]{0.195\textwidth}
        \includegraphics[width=\textwidth]{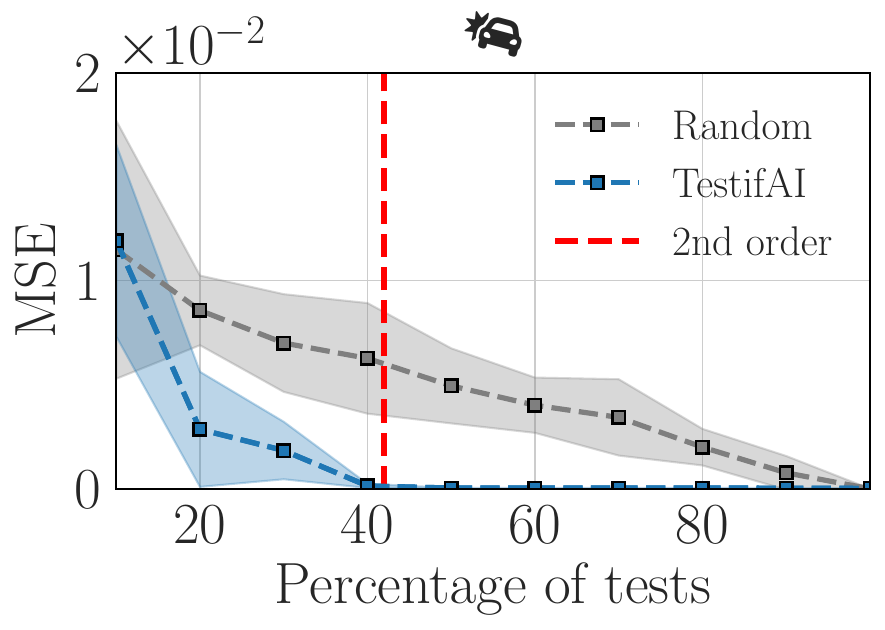}
    \end{subfigure}
    \begin{subfigure}[b]{0.195\textwidth}
        \includegraphics[width=\textwidth]{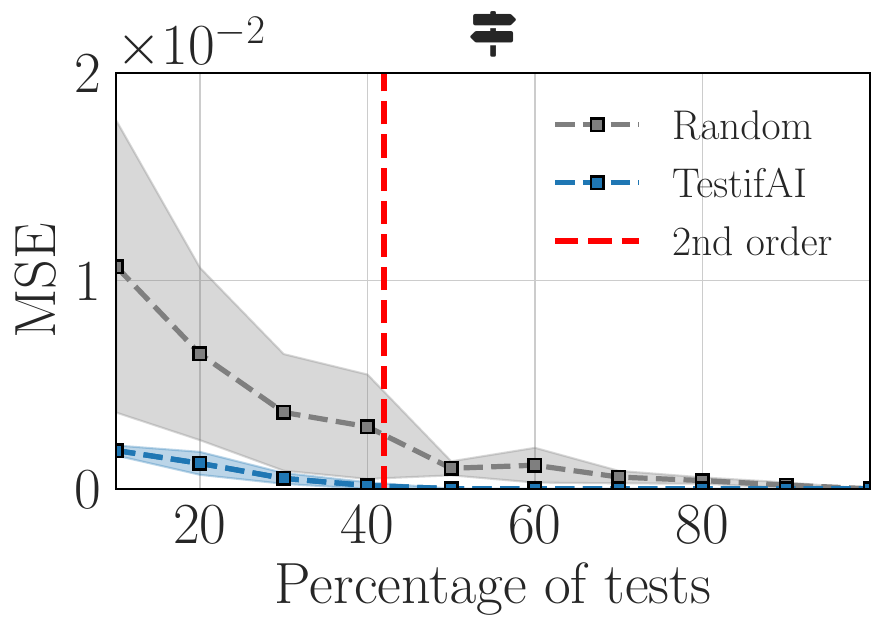}
    \end{subfigure}
    \begin{subfigure}[b]{0.195\textwidth}
        \includegraphics[width=\textwidth]{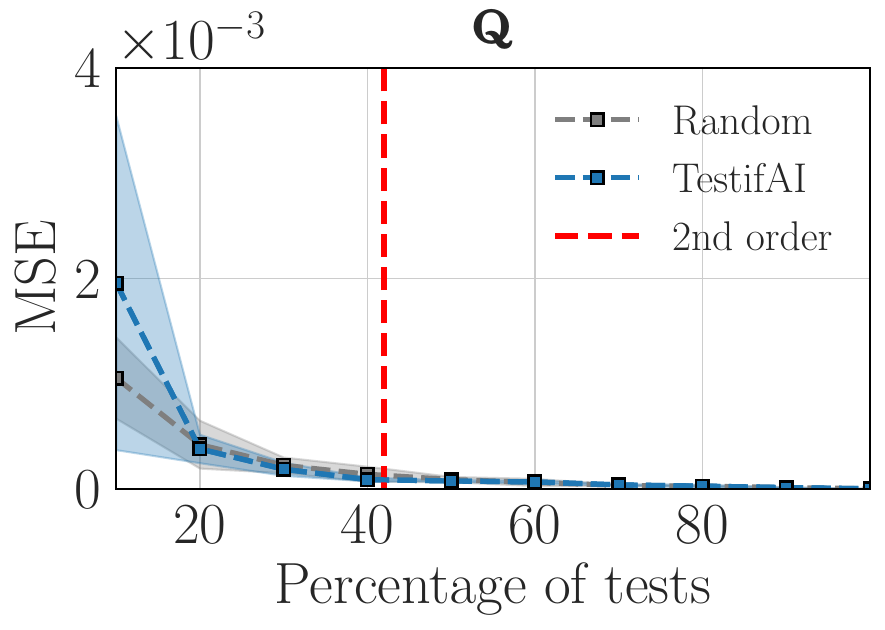}
    \end{subfigure}
    \caption{Efficiency of partial 3D model tomography.}
    \label{fig:3-perturbaitons}
\end{figure*}

\begin{figure*}[t]
    \centering
    \begin{subfigure}[b]{0.195\textwidth}
        \includegraphics[width=\textwidth]{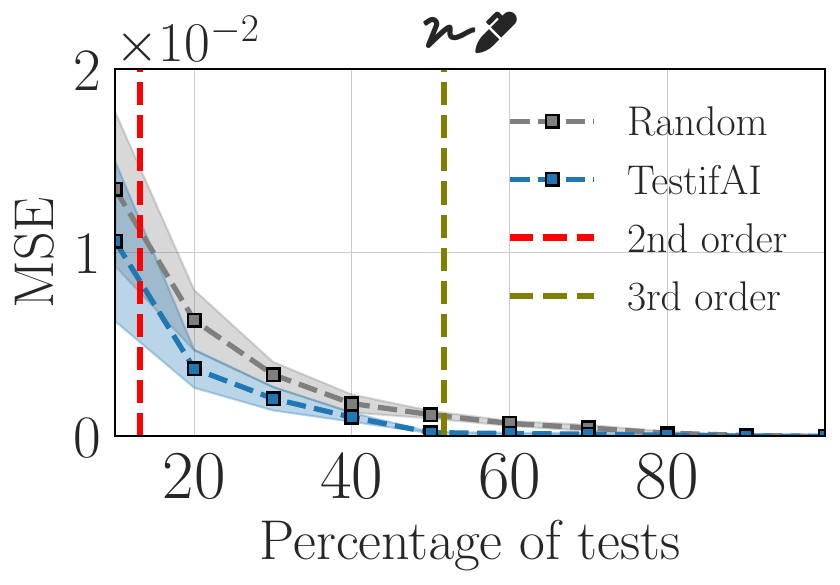}
    \end{subfigure}
    \begin{subfigure}[b]{0.195\textwidth}
        \includegraphics[width=\textwidth]{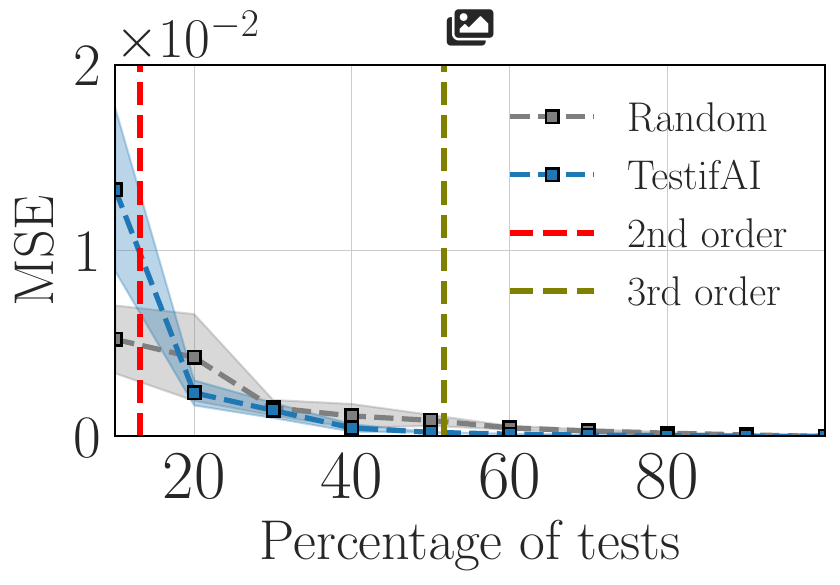}
    \end{subfigure}
    \begin{subfigure}[b]{0.195\textwidth}
        \includegraphics[width=\textwidth]{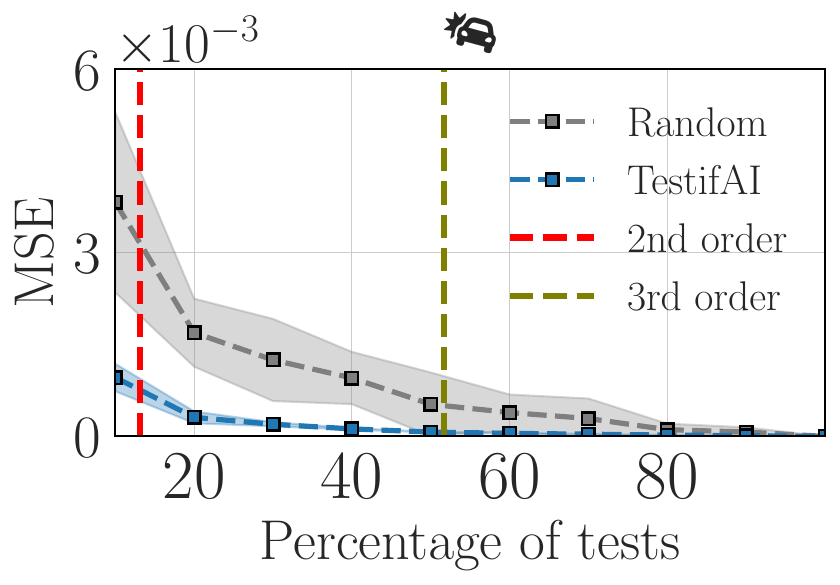}
    \end{subfigure}
    \begin{subfigure}[b]{0.195\textwidth}
         \includegraphics[width=\textwidth]{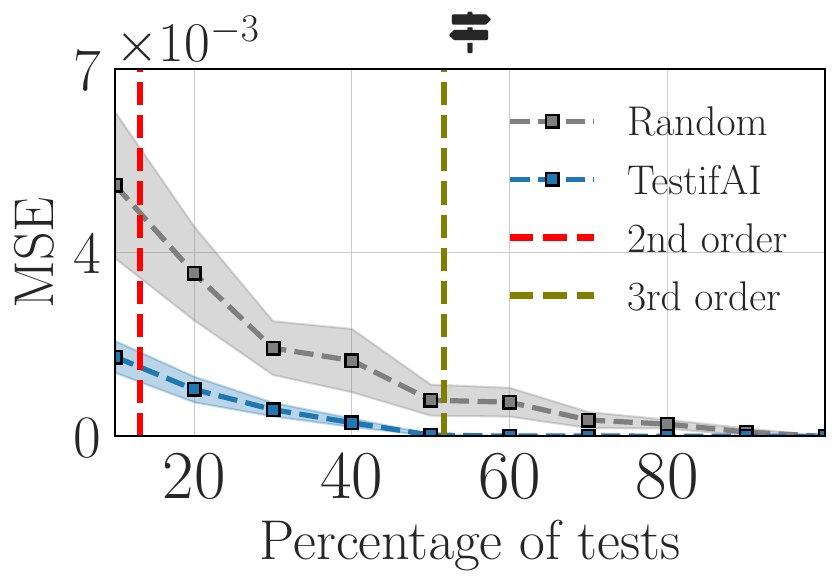}
    \end{subfigure}
    \begin{subfigure}[b]{0.195\textwidth}
         \includegraphics[width=\textwidth]{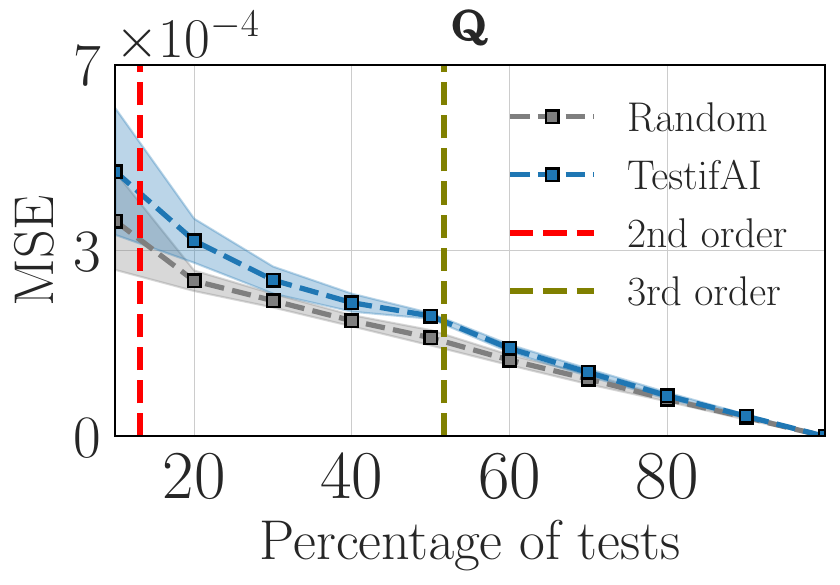}
    \end{subfigure}
    \caption{Efficiency of partial 4D model tomography.}
    \label{fig:seq-vs-random}
    
\end{figure*}

Throughout the paper, 
we 
argue
for 
using 
all {tests} in \(\Theta_{\leqslant2}\)
to 
train a model that estimates robustness 
for 
\(\Theta_{\geqslant3}\).
A natural question 
is
why not sample the same number 
of {tests} uniformly at random 
from the entire perturbation space \(\Theta\).
In this section, 
we compare these two training strategies 
and evaluate their effectiveness 
in 
accurate 
robustness estimation.

We design an experiment parametrised by two factors:
the training set size
and the strategy used
to construct it.
First, we vary the training set size
from 10\% to 100\% of 
available tests
in 10\% increments.
For each percentage level, 
we compute the corresponding 
number 
of tests ($n$) and select that many for training.
Second, we vary how the $n$ tests are selected:
\emph{(i)}  
\emph{ordered sampling} 
selects tests 
by increasing perturbations order 
(e.g., 1-way combinations, 
followed by 2-way, and so on);
and \emph{(ii)}
\emph{random sampling}
selects
tests 
uniformly 
at random from 
\(\Theta\).

For each selected test \(\boldsymbol{\sigma}\), 
we apply its perturbation configuration to every input in the dataset, 
record the model's success or failure, 
and use the results to construct the 
random forest 
training set, as described in Section~\ref{sec:sample-efficiency} and Algorithm~\ref{alg:testifai_full}.
We evaluate the model's predictions on the remaining \mbox{\(\vert \Theta \vert - n\)}
tests, 
reporting the mean squared error (MSE) 
over 
predicted 
robustness scores.
Each experiment is repeated 10 times
to account for sampling and training variability,
and we report the 
mean and standard deviation.

We run this experiment 
in
two settings: 
3D tomography 
with
\(\vert \Theta \vert = 216\) (\autoref{fig:3-perturbaitons}),
and 4D tomography 
with
\(\vert \Theta \vert = 1296\) (\autoref{fig:seq-vs-random}).
Figures~\ref{fig:3-perturbaitons} and~\ref{fig:seq-vs-random}
show MSE
as a function of training set size \(n\) 
under 
{ordered} (blue)
and {random} (gray) 
sampling.
Each point 
is averaged
over 10 runs, 
with shaded bands 
representing one standard deviation.
The red vertical dashed line 
marks
the point at which the ordered selection 
includes all second-order tests (i.e., \(\Theta_{\leqslant 2}\)), 
while the green line 
marks
inclusion of all
third-order tests 
(i.e., \(\Theta_{\leqslant 3}\)).

Across both 3D and 4D
settings,
ordered sampling 
consistently outperforms random sampling, 
achieving lower MSE 
with substantially less training data.
Ordered sampling converges rapidly: 
for most benchmarks, 
robust generalization 
is achieved with only 20–30\% 
of the perturbation space.
In
contrast, random sampling 
exhibits
slower convergence and greater variance, 
especially at small sample sizes. 
Notably,
although 
third-order data further
improves 
MSE, the 4D average MSE achieved 
using only first- and second-order data 
is already comparable with the 3D cases.

The results support our hypothesis 
that lower-order projections carry 
highly 
informative structure about 
the full perturbation space, 
and that partial model tomography 
provides a principled, 
sample-efficient alternative to random test selection.

Comparing 3D and 4D 
results, 
the 4D setting consistently 
achieves lower MSE. 
At the same sampling percentage, 
the 4D model has access 
to significantly 
more training data---for example, 
130 tests in 4D versus 
only 21 in 3D at 10\% sampling. 
This increased data density 
leads to sharper early reductions 
in MSE across all benchmarks. 

When comparing benchmarks 
within each dimensional setting, 
we observe that {\footnotesize \sdcarSymb} 
and {\footnotesize \quoraSymb} consistently yield lower MSE than {\footnotesize \mnistSymb} 
and {\footnotesize \cifarSymb}, 
despite using far fewer perturbed 
inputs to train the predictive model 
($1000$ vs. $10000$ images per test). 
This suggests 
that robustness 
estimation quality 
is driven more 
by the stability and predictability 
of 
model responses
to composite perturbations---some of which 
may have little or no effect---than 
by  
sample count
alone.
Figure~\ref{fig:query-error-third-order-stability} 
supports this argument: 
estimation error remains low 
for {\footnotesize \sdcarSymb}, {\footnotesize \cifarSymb}, 
and {\footnotesize \quoraSymb}, 
across
$Q_1$-$Q_{12}$,
whereas
{\footnotesize \mnistSymb} 
and {\footnotesize \tsignSymb} 
exhibit higher and more 
variable 
errors
even 
though their predictive model was trained with more samples. 
This is because 
the multiplicative effect 
of composite perturbations 
on 
accuracy becomes 
less predictable
as perturbation complexity and severity increase.

These results 
suggest that 
when perturbation effects
degrade 
accuracy 
in a consistent and predictable manner,
partial tomography 
can learn useful patterns 
from relatively few examples.  
This also explains the effectiveness 
of sequential testing for benchmarks with stable responses: 
early tests yield 
predictive signals that 
generalize well, 
even in higher-dimensional settings.

\subsection{Partial tomography efficiency}
\label{sec:result3}

In this experiment, we evaluate 
the execution efficiency of 
partial tomography---with and without our early 
stopping strategy---using the total number
of inferences as a proxy metric.
Since inference cost is stable 
for a given hardware setup, 
this metric provides a reliable estimate of total runtime.
We focus on 3D tomography. 

Full tomography yields, of course,
zero robustness estimation error, 
as it evaluates all tests in \(\Theta\) exhaustively. 
However, this comes at the cost of 
\(\vert \Theta \vert \times \vert \mathcal{D} \vert\) 
total inferences. 
For example, in the case 
of {\footnotesize \mnistSymb} or 
{\footnotesize \cifarSymb}, 
this requires 2,160,000 inferences.
Partial model tomography requires evaluating only up to
\(\vert \Theta_{\leqslant 2} \vert \times \vert \mathcal{D} \vert\) 
inferences, which reduces the total number of inferences 
by at least \(58\%\). 

Figure~\ref{fig:Q13_mse_bar_fullbox_labels} 
presents the total number of inferences 
required to reconstruct \(\Theta\) 
under three configurations: 
(\emph{i})~full 3D tomography; 
(\emph{ii})~partial model tomography without early stopping; 
and 
(\emph{iii})~partial tomography with early stopping.
The $x$-axis shows the benchmarks; and 
the $y$-axis shows the total number of inferences required.
Above each bar, we plot the MSE 
between the predicted and 
actual aggregate robustness scores 
for each of our random forests: one
exhaustively applies \(\Theta_{\leqslant2}\) configurations
to \textbf{all} inputs; and one to a subset.
Even with early stopping enabled, 
the MSE is small.

Without early stopping, partial tomography 
reduces the total number of inferences 
by $58\%$
for all benchmarks.
Our early stopping strategy 
(see Algorithm~\ref{alg:testifai_full}) 
further lowers the inference cost down 
to 72.7\% for {\footnotesize \mnistSymb},
61.6\% for {\footnotesize \cifarSymb},
73.15\% for {\footnotesize \sdcarSymb} and
83\% for {\footnotesize \tsignSymb}
(no additional improvements were observed for {\footnotesize \quoraSymb}).
This has a direct effect to the total computation time:
computing the full 3D tomography for {\quoraSymb} required approximately 47.3 hours, which was reduced significantly 
with partial model tomography. 
{The transformation-to-inference-time ratios are
$0.09$ for {\footnotesize \cifarSymb},
$0.20$ for {\footnotesize \sdcarSymb},
$0.30$ for {\footnotesize \quoraSymb},
$2.3$ for {\footnotesize \tsignSymb}, and
$76.7$ for {\footnotesize \mnistSymb}.
Except for {\footnotesize \tsignSymb} and {\footnotesize \mnistSymb},
inference time dominates.
Because perturbations are model-agnostic preprocessing
while inference scales with model size,
we expect inference time to dominate even more for larger models.}

\begin{figure*}[t]
    \centering
    \begin{subfigure}[b]{0.196\textwidth}
        \includegraphics[width=\textwidth]{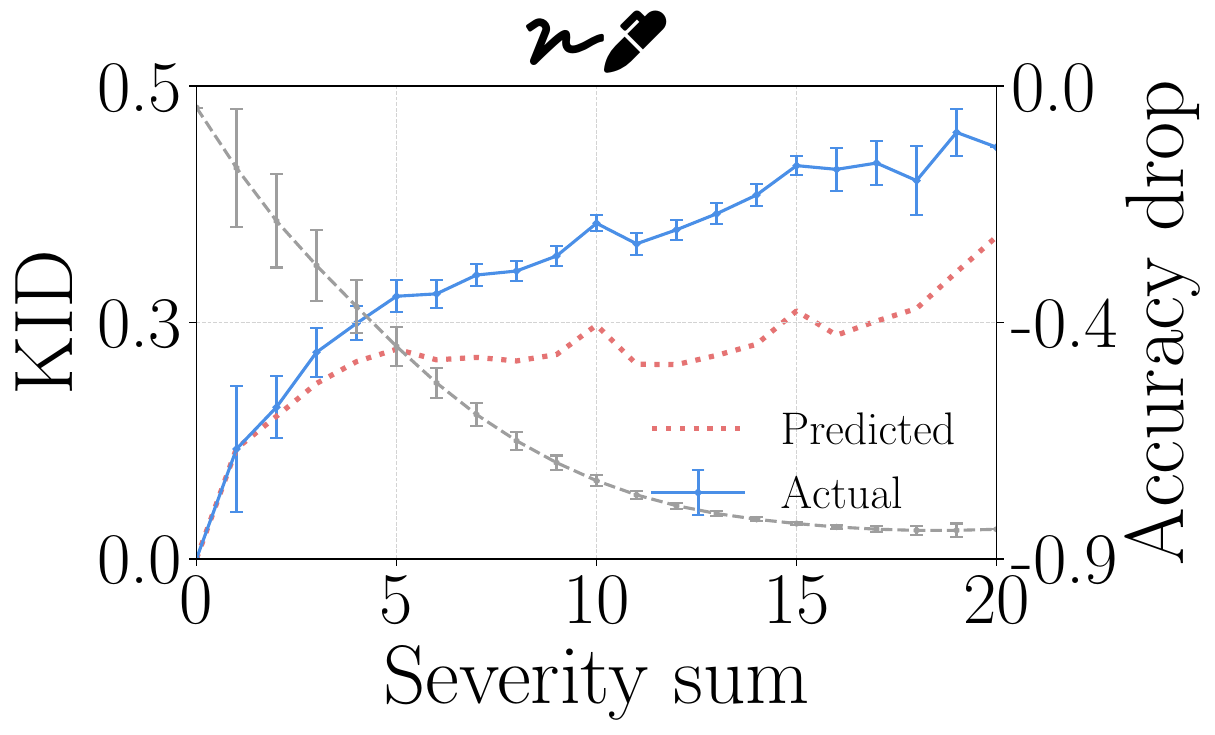}
    \end{subfigure}
    \begin{subfigure}[b]{0.196\textwidth}
        \includegraphics[width=\textwidth]{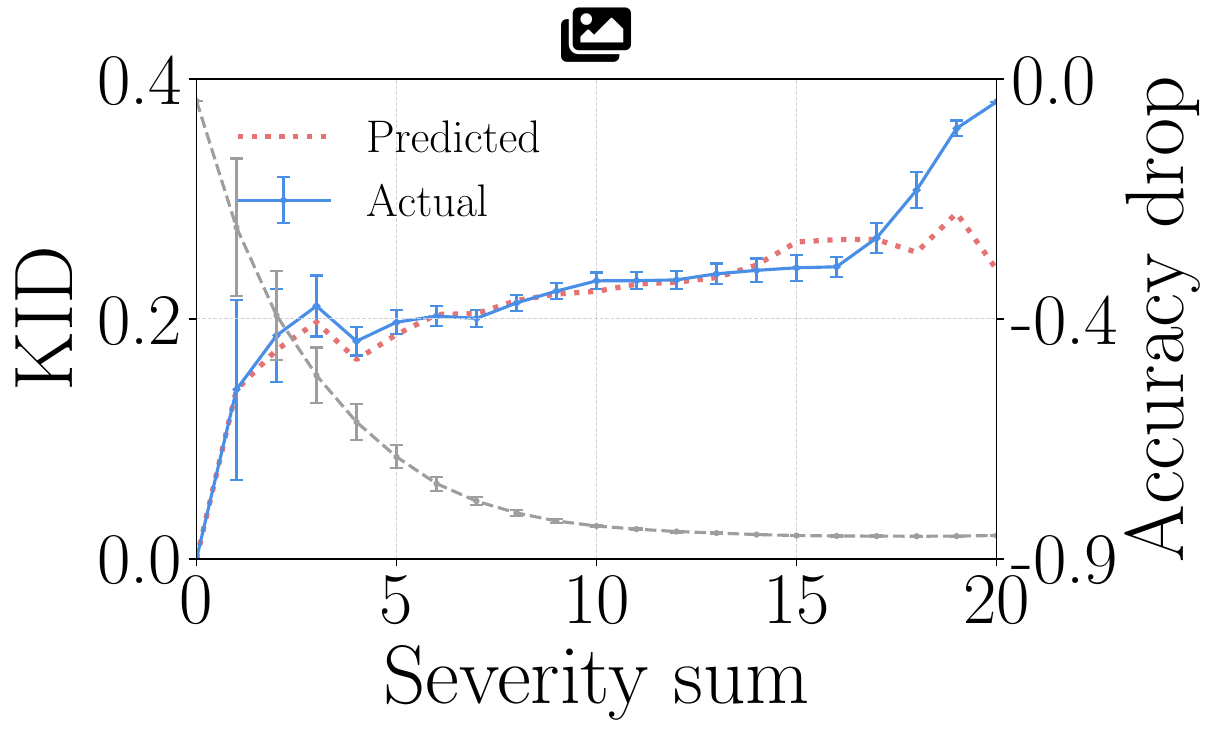}
    \end{subfigure}
    \begin{subfigure}[b]{0.196\textwidth}
        \includegraphics[width=\textwidth]{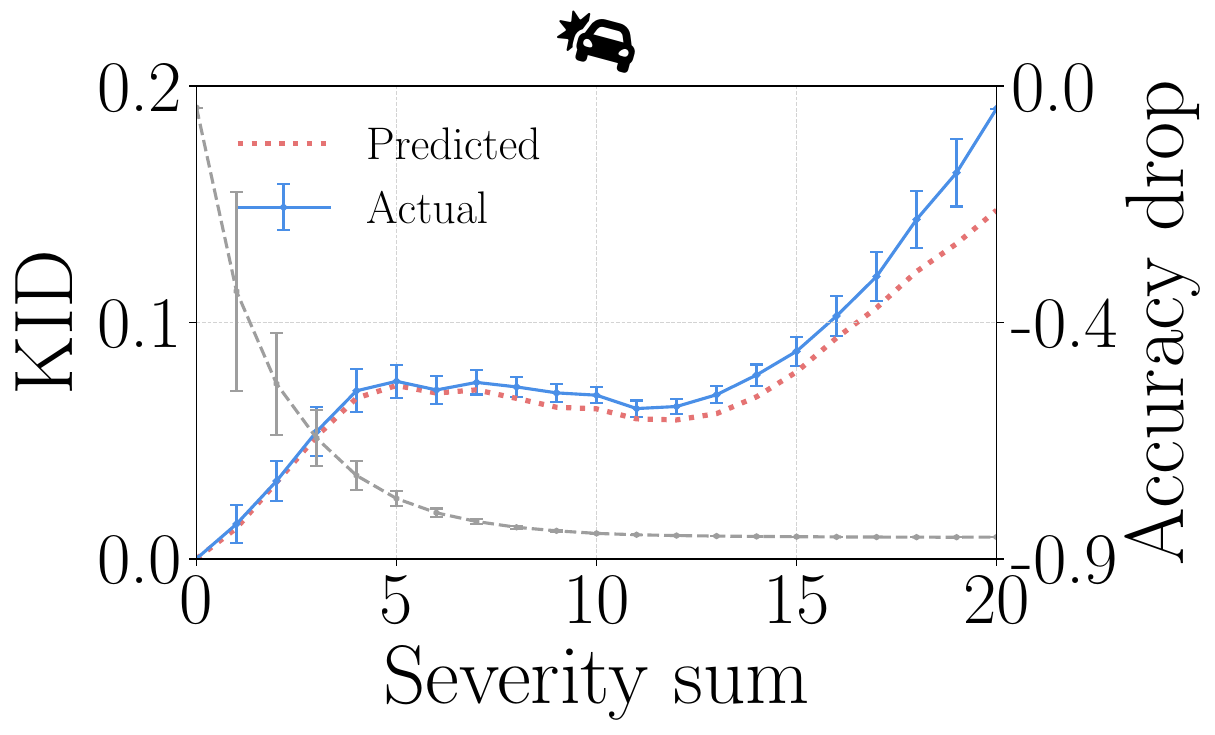}
    \end{subfigure}
    \begin{subfigure}[b]{0.196\textwidth}
        \includegraphics[width=\textwidth]{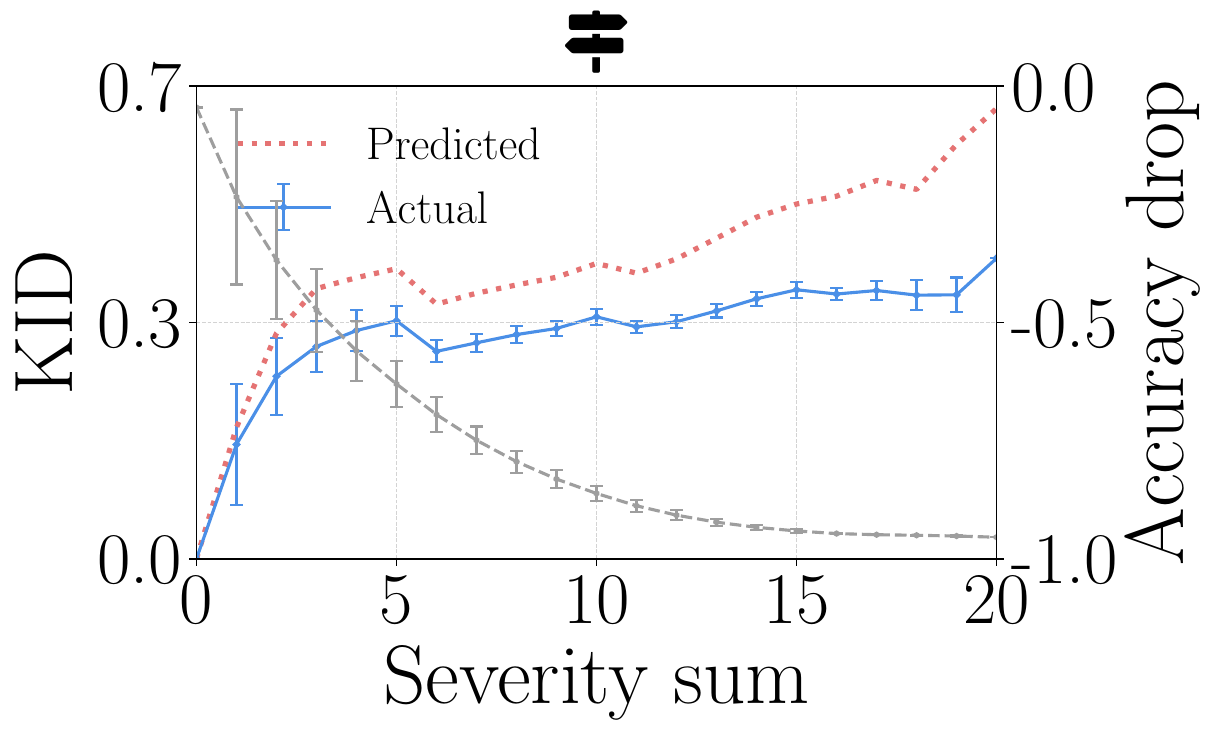}
    \end{subfigure}
    \begin{subfigure}[b]{0.196\textwidth}
        \includegraphics[width=\textwidth]{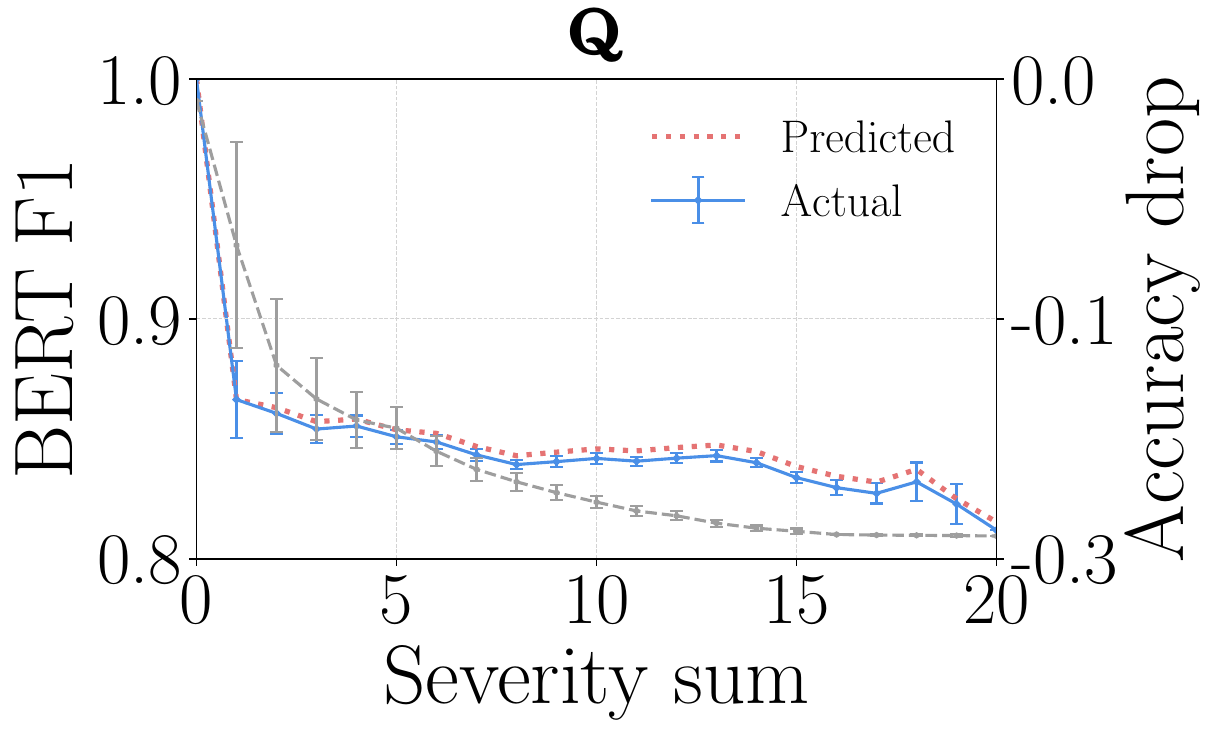}
    \end{subfigure}
    \caption{ Actual vs. predicted quality and accuracy drop under increasing perturbation severity. Blue solid lines show actual perceptual quality (KID for vision, BERT-F1 for NLP). Red dotted lines show predicted quality. Grey dashed lines (right axis) show accuracy drop. Error bars represent standard error across all combinations with the same severity sum.}
    \label{fig:quality-accuracy}
\end{figure*}

\begin{figure}[t]
    \centering
    \includegraphics[width=0.5\columnwidth]{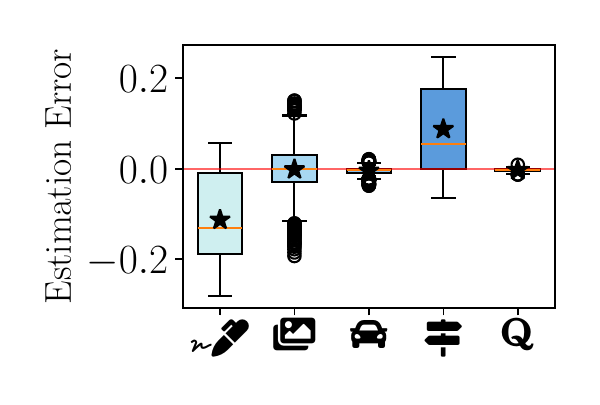}
    \caption{$\mathrm{q}_{\boldsymbol{\sigma}}$ estimation error for 4D tomography.}
    \label{fig:error-quality}
\end{figure}

\begin{figure}[t]
   
    \centering
    \includegraphics[width=0.5\columnwidth]{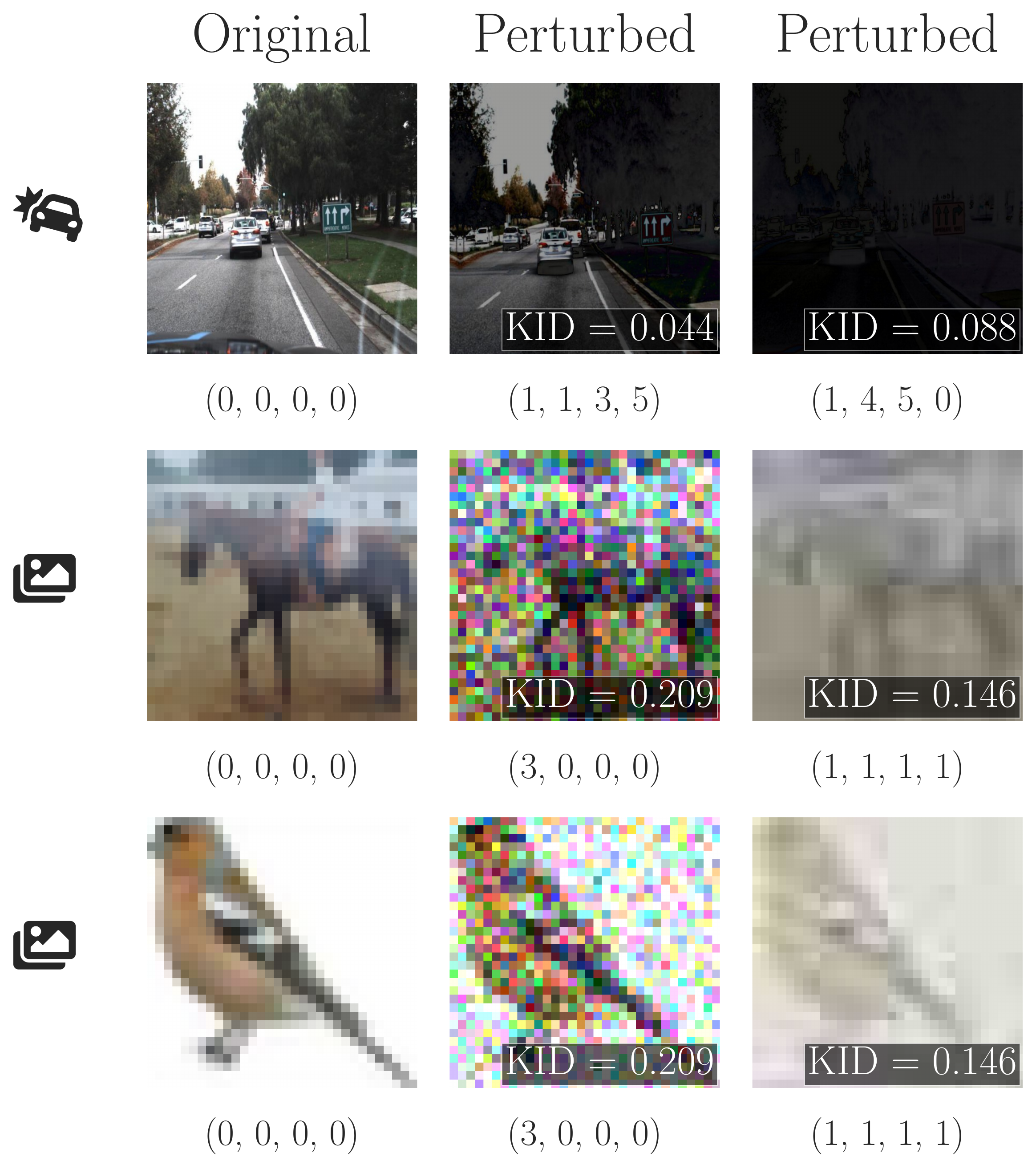}
    \caption{Example perturbed images derived from different perturbation combinations.}
    \label{fig:example-images}
\end{figure}

\section{On Validity \& Quality Estimation}
\label{sec:validity}

\sys assumes that 
domain experts 
specify 
perturbations and severity levels 
that reflect their deployment environment, 
following the practice 
of existing 
robustness benchmarks~\cite{mu2019mnistc,hendrycks2019cifar10c,olivesgatech_CURETSD}. 
As such, \sys does not 
attempt to 
assess the semantic validity of 
higher-order perturbations, or, more generally,
of the perturbation space \(\Theta\)
within 
a model's application domain.

Nevertheless, 
it is informative 
to distinguish 
robustness failures
under \emph{plausible} 
perturbations 
from
those caused by \emph{invalid} 
or highly unrealistic 
ones.
To this end, we 
use
quality metrics 
as proxies for perturbation validity.
Specifically, 
we can use 
the \sys framework to estimate 
the effect of higher-order perturbations on 
input quality,
and then refine testing by augmenting user queries with constraints 
on these metrics. 
We next present 
a quantitative validity analysis of our 
five benchmarks and 
show
how \sys can estimate 
higher-order validity 
from low-order observations 
in the same way it estimates robustness.

For our four vision tasks, 
we use 
Kernel Inception Distance (KID)~\cite{BinkowskiSAG18}
to measure image realism, 
and for the language task, we use BERTScore~\cite{zhang2020bertscore}
to measure 
semantic preservation.
These metrics quantify 
deviation 
from the distribution of
unperturbed inputs
and serve as proxies for perturbation validity. 
Although
we also 
evaluated
other 
quality metrics---including FID and SSIM 
for images and BLEU for text---we found 
that KID and BERTScore integrate best 
with partial tomography and early stopping; 
in particular, KID remains reliable 
under partial sampling of the input distribution. 

We collected the true quality 
score~$\mathrm{q}_{\boldsymbol{\sigma}}$ 
for every test~$\boldsymbol{\sigma}$ 
in the full 4D perturbation space 
(1296 in total).
We then applied partial tomography to estimate 
higher-order
quality scores 
from lower-order ones, 
leveraging the fact that
partial tomography 
is
\emph{agnostic} to the nature 
of the predicted score.
We trained a random-forest \emph{regressor} 
on \(\Theta_{\leqslant 2}\),
replacing
robustness scores $\mathrm{r}_{\boldsymbol{\sigma}}$  
with 
$\mathrm{q}_{\boldsymbol{\sigma}}$.
We used the same feature representation 
as for robustness 
prediction. 

Figure~\ref{fig:quality-accuracy} 
shows 
quality 
(\({\mathrm{q}}_{\boldsymbol{\sigma}}\)), 
predicted quality 
(\(\hat{\mathrm{q}}_{\boldsymbol{\sigma}}\)), 
and accuracy 
drop as functions of the 
\emph{severity sum}
$\sigma_1 + \dots + \sigma_k$ for $\boldsymbol{\sigma}=(\sigma_1,\dots,\sigma_k)$---that is,
the 
$\ell_{1}$-norm of the 
perturbation severity vector
(a slice of the 4D 
perturbation space 
indexed by total severity). 
A given severity sum 
can arise from either 
a 
few 
large perturbations 
or 
the compounding effects 
of many small ones.
\autoref{fig:error-quality} shows that 
across 
benchmarks 
the mean absolute estimation error 
is approximately~$0.1$ 
for {\footnotesize \mnistSymb} 
and {\footnotesize \tsignSymb}, 
and near zero for 
{\footnotesize \cifarSymb}, 
{\footnotesize \sdcarSymb}, 
and {\footnotesize \quoraSymb}, 
indicating that validity, like robustness, 
is 
structured and predictable 
from low-order observations.

A user concerned with application-level 
validity can integrate these scores 
directly into the \sys querying mechanism: 
given a threshold $\tau_{\mathrm{q}}$ on predicted 
input quality, \sys can exclude all tests 
with predicted (or measured) quality 
above~$\tau_{\mathrm{q}}$ 
(or below, in the case of BERT~F1). 
This enables queries of the form
\(
Q \;\land\; 
\hat{\mathrm{q}}_{\boldsymbol{\sigma}} \le \tau_{\mathrm{q}},
\)
where~$Q$ is any 
existing Boolean constraint 
(\S\ref{sec:user-interactions}). 
The threshold further allows users to 
distinguish errors 
under plausible conditions 
($\hat{\mathrm{q}}_{\boldsymbol{\sigma}} \le \tau_{\mathrm{q}}$) 
from those attributable to 
invalid inputs 
($\hat{\mathrm{q}}_{\boldsymbol{\sigma}} > \tau_{\mathrm{q}}$).

For example, 
the top row of 
\autoref{fig:example-images} 
shows two perturbed versions of 
the same image from {\footnotesize \sdcarSymb} 
at severity sum~$10$. 
The fourth-order
perturbation $(1, 1, 3, 5)$
with KID~$0.043$ 
is likely considered 
valid input, 
whereas the third-order
one, 
with KID~$0.088$, 
is likely considered invalid. 
By inspection, 
KID scores up to roughly~$0.06$ 
appear generally valid 
for {\footnotesize \sdcarSymb}. 
Comparing this threshold with 
\autoref{fig:quality-accuracy} 
indicates that similarly low 
KID scores can still occur 
even at severity sums as high as~$14$ 
for {\footnotesize \sdcarSymb}.

Examples 
from {\footnotesize \cifarSymb} 
in the bottom two rows of 
Fig.~\ref{fig:example-images} 
show that equal KID scores 
can correspond to both plausible 
(identifiable) and invalid 
(unidentifiable) inputs.
They also show that the compounding 
effects of different perturbations 
may differ from applying a single 
perturbation at higher severity, 
leading to cases where images with 
larger KID scores are more identifiable 
than those with lower scores 
near the boundary of validity.
The (im)plausibility of an 
input perturbation, therefore, 
depends on the 
data set, the chosen perturbations and 
their strengths, and the application 
domain.

These examples illustrate how users can refine \sys 
using
validity or other system-level metrics. 
Choosing 
appropriate thresholds
requires 
domain expertise, 
ranging from manual inspection 
to deployment-specific analysis. 
Importantly, 
high severity sums do not necessarily 
imply invalid inputs: 
multiple perturbations often 
co-occur in practice (e.g., combined weather and 
sensor effects in vision~\cite{mușat2021multi}, 
or surface-form perturbations in text~\cite{wang2023recode,hao2024your}).
Accordingly, \sys does not discard severe cases 
\emph{a priori}; instead, it surfaces them  
with 
robustness and validity estimates, 
enabling
users to judge whether 
they fall 
within 
acceptable operational bounds.

\section{Discussion}

\sys assumes a user-defined 
space~\(\Theta\) of \emph{semantic} perturbations, 
each with a small discrete 
set of severity levels, making it 
queryable 
for partial tomography.
Even when perturbations have 
continuous severity parameters 
or produce similar effects 
(e.g., fog and contrast), 
coarse binning 
often suffices
to capture robustness trends. 
However, if no semantically meaningful 
discretization exists 
(e.g., gradient-based adversarial attacks~\cite{wan2023average,guo2021gradient}), 
\sys does not apply; 
it is best used as 
an exploratory tool for refining a given 
perturbation family 
(e.g., weather corruptions in ImageNet-C~\cite{hendrycks2019benchmarking} 
or paraphrasing in TextFlint~\cite{wang2021textflint}) 
rather than discovering them.

\sys assumes 
that regions 
\(\Theta_{\geqslant 3}\) are 
sufficiently structured 
that robustness measurements 
in \(\Theta_{\leqslant 2}\) 
capture most interactions. 
When models exhibit 
higher-order effects 
that are not predictable 
from single or pairwise 
tests, 
\sys 
may
misestimate 
robustness. 
Figure~\ref{fig:seq-vs-random} shows that 
adding a modest number of 
\(\Theta_{3}\) oracle evaluations 
improves prediction. 
When practitioners suspect 
higher-order coupling, 
allocating a small 
\(\Theta_{3}\) budget can reduce 
misestimation. 
Automatically identifying 
such cases is future work.

\sys enables 
multi-perturbation analysis 
across diverse model types 
and modalities. 
Partial tomography requires 
a discretizable perturbation 
space and 
a system-level metamorphic test 
to access 
whether behaviour 
under perturbation 
remains acceptable. 
These 
requirements
extend beyond classification tasks.
E.g., \S\ref{sec:validity}
shows
how partial tomography 
can apply to regression and generation: 
KID and BERTScore 
serve as quality-based oracle signals 
for perturbed images and text; 
and similar principles 
enable testing of
code-generation models~\cite{wang2023recode} 
using CodeBLEU~\cite{ren2020codebleu}.

\section{Related Work}

Prior work on testing deep learning models spans several directions,
including 
training input generation (e.g., \textsc{AugMix}~\cite{hendrycks2020augmix})
coverage-guided testing (e.g., DeepXplore~\cite{pei2017deepxplore}),
test prioritisation (e.g., DeepGini~\cite{feng2020deepgini}), and
combinatorial testing (e.g., CIT4DNN~\cite{dola2024cit4dnn}).
While these methods provide 
valuable insights into model behaviour,
they 
neither
support systematic 
reasoning over structured multi-perturbation spaces,
nor 
enable 
users to query robustness 
under diverse, interacting perturbations.

\emph{Benchmarks.} 
MNIST-C~\cite{mu2019mnistc},
CIFAR-10-C and ImageNet-C~\cite{hendrycks2019benchmarking},
KITTI-C, nuScenes-C, and Waymo-C~\cite{dong2023benchmarking}, 
and CURE-TSR~\cite{cure-tsr}
are benchmarks that 
define curated sets of meaningful perturbations 
to test the robustness of computer vision models. 
Similarly, \textsc{CheckList}~\cite{ribeiro2020beyond} 
and \textsf{TextFlint}~\cite{wang2021textflint} 
provide linguistic perturbations 
for testing natural language models.
These are complementary resources to \sys. Users 
can draw perturbations from them 
to systematically explore 
and test their combinations.

\emph{Data augmentation methods.} 
CutMix~\cite{yun2019cutmix}, 
AugMix~\cite{hendrycks2020augmix}, 
and PixMix~\cite{hendrycks2022pixmix}) aim 
to \emph{improve} 
the robustness of deep learning models
applying stochastic combinations of perturbations
during training.
These augmentations expose 
models to inputs 
affected by multiple, 
simultaneously 
applied transformations (e.g., blur, contrast, rotation).
\sys rather aims to systematically 
\emph{explore} 
and \emph{estimate} 
the aggregate robustness of a trained model
under structured, multi-perturbation 
test environments.

\emph{Neuron coverage.} 
DeepXplore~\cite{pei2017deepxplore}, DeepGuage~\cite{ma2018deepgauge}, DLFuzz~\cite{dlfuzz} and DeepHunter~\cite{xie2019deephunter}
aim to uncover erroneous behaviours in deep learning models 
by generating test inputs that 
maximise neuron activation coverage. 
However, increased coverage does not necessarily
correlate with a higher rate of error discovery,
and often leads to the generation of less natural or semantically meaningful inputs~\cite{riccio2023invalid,harel2019neuron}.
These methods typically produce new inputs
by applying pixel-level perturbations to a small set of seed examples,
which limits both the diversity and semantic fidelity of the resulting test data~\cite{dola2024cit4dnn}.
In contrast, \sys generates test inputs
by systematically combining interpretable 
semantic perturbations.

\emph{Test prioritization.}
Input prioritization techniques aim to identify test inputs
that are more likely to reveal model errors,
thereby accelerating the discovery of misclassifications~\cite{guo2025coverage,weiss2022simple}.
DeepGini~\cite{feng2020deepgini}, for instance, prioritizes inputs
on which the model exhibits low confidence (i.e., high uncertainty in softmax output).
In contrast, \sys prioritises system-level metamorphic tests—structured combinations
of semantic perturbations—rather than individual inputs.
This enables efficient estimation of robustness across the entire perturbation space.
Input-level prioritization is orthogonal to our work
and could complement \sys's early stopping strategy,
particularly when users seek to identify worst-case robustness.

\emph{Combinatorial Interaction Testing (CIT).}
CIT has been applied 
to deep learning in several ways, 
differing along two axes: 
model access---whether
\emph{white-box}~\cite{ma2019deepct,chen2019variable} 
or \emph{black-box}~\cite{dola2024cit4dnn,chandrasekaran2021combinatorial}---and 
the type of interaction under test---neuron activations~\cite{ma2019deepct,chen2019variable}, 
latent input features~\cite{dola2024cit4dnn}, 
or high-level semantic perturbations~\cite{chandrasekaran2021combinatorial}).

DeepCT~\cite{ma2019deepct}, for example, is a white-box method that
systematically explores neuron combinations within a layer
to maximise combinatorial activation coverage
and, thus, uncover more faults---a test method 
aligned with neuron coverage methods discussed earlier. 
CIT4DNN~\cite{dola2024cit4dnn}, in contrast, 
is a black-box method: it learns
a latent representation of the inputs,
and then applies CIT on latent dimensions
to generate rare or
diverse inputs.
Neither method guarantees that the generated
test inputs can be interpreted 
as semantically meaningful perturbations.

Chandrasekaran~\emph{et al.}~\citep{chandrasekaran2021combinatorial}
generate two-way combinations of common image perturbations
(e.g., blur, brightness, and rotation) 
applied to a small, curated input set
to produce synthetic driving scenes for robustness testing.
\sys extends this idea beyond pairwise testing:
it uses second-order tests
to approximate higher-order robustness behaviour
through 3D and 4D tomography.

\sys departs from CIT in a fundamental way:
it does not attempt to construct high-dimensional covering arrays,
as done in prior work on classical software systems 
(e.g., ScalableCA~\cite{luo2024beyond}).
Instead, it leverages second-order tests to train a predictive model
that estimates robustness in multi-perturbation spaces.

\emph{Formal guarantees.}
\sys{}'s 
system-level metamorphic tests 
assess the extent to which a model is 
robust under \emph{some} structured, 
semantic perturbations of a set of inputs. 
This differs from 
\emph{local robustness}, 
where neural-network verifiers 
(e.g., $\upalpha\upbeta$-CROWN~\cite{zhou2024scalable}) 
check for
\emph{all} $\ell_p$-bounded 
perturbations of a given input.
Likewise, \sys{}'s 
aggregate robustness is an 
empirical summary over a semantic 
region~$\Theta$, 
not a \emph{global robustness} certificate 
that requires local robustness of 
\emph{all} possible inputs 
(as in fairness~\cite{khedr2023certifair} and global robustness 
certification~\cite{wang2022efficient,kabaha2024verification}).
These perspectives are complementary: 
\sys finds and ranks 
semantically meaningful regions 
for further analysis, 
while formal verification 
provides worst-case guarantees 
on 
selected subsets.

\section{Conclusion}

\sys reframes deep learning testing
as a predictive modelling task, reasoning 
about how multiple perturbations interact, rather 
than a combinatorial coverage problem.
It trains a surrogate model using 
lower-order 
test results---specifically, 
all combinations of up to two perturbations 
and their severity levels---to
approximate 
robustness 
over 
higher-order tests involving three or more perturbations.
Experiments on five vision 
and language benchmarks 
show that \sys 
predicts third- and fourth-order 
test outcomes
from second-order observations 
with
less than 7\%
aggregate estimation error.
\sys is a step towards 
\emph{test-oriented generalization},
enabling principled extrapolation of test results
to 
unseen multi-perturbation scenarios.
It helps users anticipate model behaviour
under complex, high-dimensional perturbations
beyond what pairwise (2-way) testing can expose.

\balance
\bibliographystyle{ACM-Reference-Format}
\bibliography{bibliography}

\end{document}